%% file: access.tex
\documentclass{ieeeaccess}
\usepackage{cite}
\usepackage{amsmath,amssymb,amsfonts}
\usepackage{graphicx}
\usepackage{textcomp}
\usepackage{algorithm}
\usepackage{algpseudocode}
\usepackage{caption}
\usepackage{bm}
\usepackage{booktabs}
\usepackage{multirow}
\usepackage{placeins}
\usepackage{amsthm}
\usepackage[normalem]{ulem}
\usepackage{soul}

\usepackage[commentmarkup=footnote,defaultcolor=black,highlightmarkup=background,final]{changes}

\input{math_commands}

\input{changes-defs}

\makeatletter
\AtBeginDocument{\DeclareMathVersion{bold}
\SetSymbolFont{operators}{bold}{T1}{times}{b}{n}
\SetSymbolFont{NewLetters}{bold}{T1}{times}{b}{it}
\SetMathAlphabet{\mathrm}{bold}{T1}{times}{b}{n}
\SetMathAlphabet{\mathit}{bold}{T1}{times}{b}{it}
\SetMathAlphabet{\mathbf}{bold}{T1}{times}{b}{n}
\SetMathAlphabet{\mathtt}{bold}{OT1}{pcr}{b}{n}
\SetSymbolFont{symbols}{bold}{OMS}{cmsy}{b}{n}
\renewcommand\boldmath{\@nomath\boldmath\mathversion{bold}}}
\makeatother

\def\BibTeX{{\rm B\kern-.05em{\sc i\kern-.025em b}\kern-.08em
    T\kern-.1667em\lower.7ex\hbox{E}\kern-.125emX}}

\begin{document}
\history{Date of publication xxxx 00, 0000, date of current version xxxx 00, 0000.}
\doi{10.1109/ACCESS.2024.0429000}

\title{Compressing What Matters: Neuron Importance Meets Data-Aware Low Rank Approximation for Language Model Compression}
\author{\uppercase{Athanasios Ntovas}\authorrefmark{1}, \uppercase{Alexandros Doumanoglou} \authorrefmark{1}, \uppercase{Petros Drakoulis}\authorrefmark{1} and \uppercase{Dimitris Zarpalas}\authorrefmark{1}}

\address[1]{
Information Technologies Institute (ITI), Centre for Research and Technology HELLAS (CERTH), Thessaloniki, Greece. \\(email: \{atdovas,aldoum,petros.drakoulis,zarpalas\}@iti.gr)
}

\tfootnote{This research was supported by the EU Project VOXReality (Voice-driven Interaction in XR Spaces) under Grant 101070521.}

\markboth
{Ntovas \headeretal: Preparation of Papers for IEEE TRANSACTIONS and JOURNALS}
{Ntovas \headeretal: Preparation of Papers for IEEE TRANSACTIONS and JOURNALS}

\corresp{Corresponding author: Athanasios Ntovas (e-mail: atdovas@iti.gr).}

\begin{abstract}
To excel at their domain, large language models are comprised of billions of parameters. Yet, this comes at the cost of huge memory requirements, restricting their applicability in resource-constrained environments. To address the problem of neural network (NN) compression, Singular Value Decomposition (SVD) has played a key role as a fundamental component for matrix compression through decomposition. To minimize compression error and to maximize the efficacy of the compressed model on the downstream tasks, previous works focused on low-rank approximation of the NN’s weight matrices either from the perspective of parameter importance or per-layer functional equivalence. While previous works studied the aforementioned perspectives in isolation, in this work, we are investigating the effectiveness of an approach that combines ideas from these two perspectives in a single objective. In parallel to this, an important aspect that affects the compression quality is the distribution of the compression rate across layers and NN parameters. Earlier works mostly considered distributing the compression rate uniformly across layers and network weights or relied on computationally expensive heuristic search. Contrary to them, in this work, we propose an enhanced and computationally efficient algorithm for dynamic compression rate allocation. Experimental results support the efficacy of the proposed approach, which performs on par or substantially better than the previous state-of-the-art, especially under high compression ratios.
\end{abstract}


\titlepgskip=-21pt

\maketitle

\section{Introduction}
\label{sec:introduction}
\input{sections/introduction}

\section{Related Work}
\label{sec:related-work}
\input{sections/related-work}

\section{Background}
\label{sec:background}
\input{sections/background}

\section{Motivation}
\label{sec:motivation}
\input{sections/motivation}

\section{Proposed Method}
\label{sec:method}
\input{sections/method}

\section{Experimental Results}
\label{sec:experimental-results}
\input{sections/experiments}

\section{Conclusion}
\label{sec:conclusion}
\input{sections/conclusion}

\bibliographystyle{IEEEtran}
\bibliography{refs}

\input{bios/atdovas/bio}
\input{bios/aldoum/bio}
\input{bios/drak/bio}
\input{bios/zarpalas/bio}

\EOD

\end{document}

%% file: math_commands.tex
\usepackage{amsmath,amsfonts,bm}

\def\eqref#1{equation~\ref{#1}}

\def\1{\bm{1}}

\def\vw{{\bm{w}}}
\def\vx{{\bm{x}}}

\def\mA{{\bm{A}}}
\def\mB{{\bm{B}}}
\def\mC{{\bm{C}}}
\def\mD{{\bm{D}}}

\def\mI{{\bm{I}}}

\def\mK{{\bm{K}}}

\def\mQ{{\bm{Q}}}

\def\mS{{\bm{S}}}

\def\mU{{\bm{U}}}
\def\mV{{\bm{V}}}
\def\mW{{\bm{W}}}
\def\mX{{\bm{X}}}
\def\mY{{\bm{Y}}}

\def\mSigma{{\bm{\Sigma}}}

\DeclareMathAlphabet{\mathsfit}{\encodingdefault}{\sfdefault}{m}{sl}
\SetMathAlphabet{\mathsfit}{bold}{\encodingdefault}{\sfdefault}{bx}{n}

\newcommand{\E}{\mathbb{E}}

\newcommand{\R}{\mathbb{R}}



%% file: changes-defs.tex
\sethlcolor{yellow}

\setaddedmarkup{\color{black}{\hl{#1}}}
\setdeletedmarkup{\color{red}{{\st{#1}}}}

\DeclareRobustCommand{\rcite}[1]{\protect\cite{#1}}
\soulregister\rcite7

%% file: sections/introduction.tex
\PARstart{T}{he} advances in Artificial Intelligence (AI) over the last decade have enabled a multitude of applications that were previously considered unreachable using traditional programming techniques. Nowadays, AI powers computer vision applications that can understand our world and chatbots that can communicate with humans in natural language, among others. Currently, AI is mainly realized by deep neural networks (DNNs), which are computational architectures that are trained to solve tasks from a particular domain at human-level performance. Large Language Models (LLMs) \cite{gpt,bert,gemma2,gemma3,llama,claude} constitute DNNs that excel in Natural Language Understanding (NLU) and Processing (NLP) and enable fluent human-machine interfaces based on natural language. However, these DNNs come with a ton of parameters and a large memory and computational footprint. With AI becoming ubiquitous, their efficient deployment in resource-constrained environments such as edge devices, smartphones, and mini-PCs becomes a hard requirement. Thus, compressing LLMs for efficient deployment has naturally emerged as a field of its own \cite{comp-survey1,comp-survey2}.

Since 2017, the transformer architecture \cite{transformer} has been the main workhorse behind LLMs. In their simplest form, transformers process natural language by splitting sentences and words into tokens, which are subsequently processed sequentially by a series of transformer blocks, organized in layers. Each transformer block consists of a self-attention module and a feedforward layer, both of which are parameterized by matrices of weights, with each weight corresponding to a parameter that is tuned during network training. Compressing a transformer network can be accomplished in a variety of ways, but the end goal is common: to reduce the number of network parameters that need to be stored in memory with a minimal compromise in network performance.

\input{sections/figures/architecture}

One line of research for LLM compression attempts to reduce the LLM's parameters by approximating the weight matrices of the transformer blocks with the product of two low-rank matrices, a technique also known as \textit{low-rank approximation} . Since low-rank, these matrices are crafted to have a sum of parameters that is less than the parameters found in the original matrix. Singular Value Decomposition (SVD) \cite{svd} is a matrix decomposition method \added{\rcite{comp-survey3}} that plays a key role in the low-rank approximation problem, as, given a target rank, it is proven to produce the best possible approximation with a minimum matrix reconstruction error.

Beyond its naive use, SVD has been employed in more sophisticated approaches for compressing DNN matrices. First, using weighted low-rank factorization, FWSVD \cite{fwsvd} takes into account the importance of each parameter to the DNN's output, improving model performance under equivalent compression ratios. And second, the data-aware approach of SVD-LLMv2 \cite{svd-llmv2} aims for a functionally equivalent approximation of the DNN's weight matrix by taking into account layer inputs to be multiplied with the transformation matrix through a small calibration dataset. Each one of the approaches has demonstrated significant improvements compared to baselines, yet an approach combining ideas from both is currently missing from the literature. 

In this work, our first contribution is to address this gap. Since data-aware low rank approximation has been the previous state-of-the-art in LLM compression, we opt for keeping it as a fundamental component of our method while seeking to improve it. Our investigation and analysis show that naively combining data-aware low rank approximation with the \textbf{parameter} importance estimation method that was suggested by FWSVD degrades the LLM's performance on downstream tasks compared to using data-aware low rank approximation alone. To overcome this, we propose to use \textbf{neuron} importance estimation, an alternative approach which considers parameters in functional groups, that often leads to substantial improvements in downstream model performance compared to the previous state-of-the-art under similar compression ratios.

Besides the contributions that we make in the compression method itself, we also consider the problem of parameter count allocation across weight matrices. Although in previous works uniform allocation has been the most common approach to distribute parameters across matrices, in the recent approach of SVD-LLMv2, a novel algorithm has been proposed as an improvement of the previous naive uniform strategy. This previous approach relies on grouping weight matrices based on their functionality in the transformer blocks. In this work, we make a second contribution by providing experimental evidence that a different grouping strategy, based on layer index, is more effective. Our analysis, ablation studies, and experimental findings support that when compressing BERT \cite{bert} \deleted{one of the most popular and fundamental LLMs}, \added{DistilBERT \rcite{distilbert}, MobileBERT  \rcite{mobilebert}, and TinyBERT \rcite{tinybert}, foundational models} for natural language understanding, in the vast majority of cases, the proposed approach achieves state-of-the-art performance compared to other SVD-based algorithms. \added{Furthermore, we include a computational analysis across these compressed architectures (DistilBERT, MobileBERT, and TinyBERT) detailing the significant reduction in both MFLOPS/token and total parameter count.}

%% file: sections/figures/architecture.tex
\begin{figure*}[t] 
    \centering
    \includegraphics[width=1\textwidth]{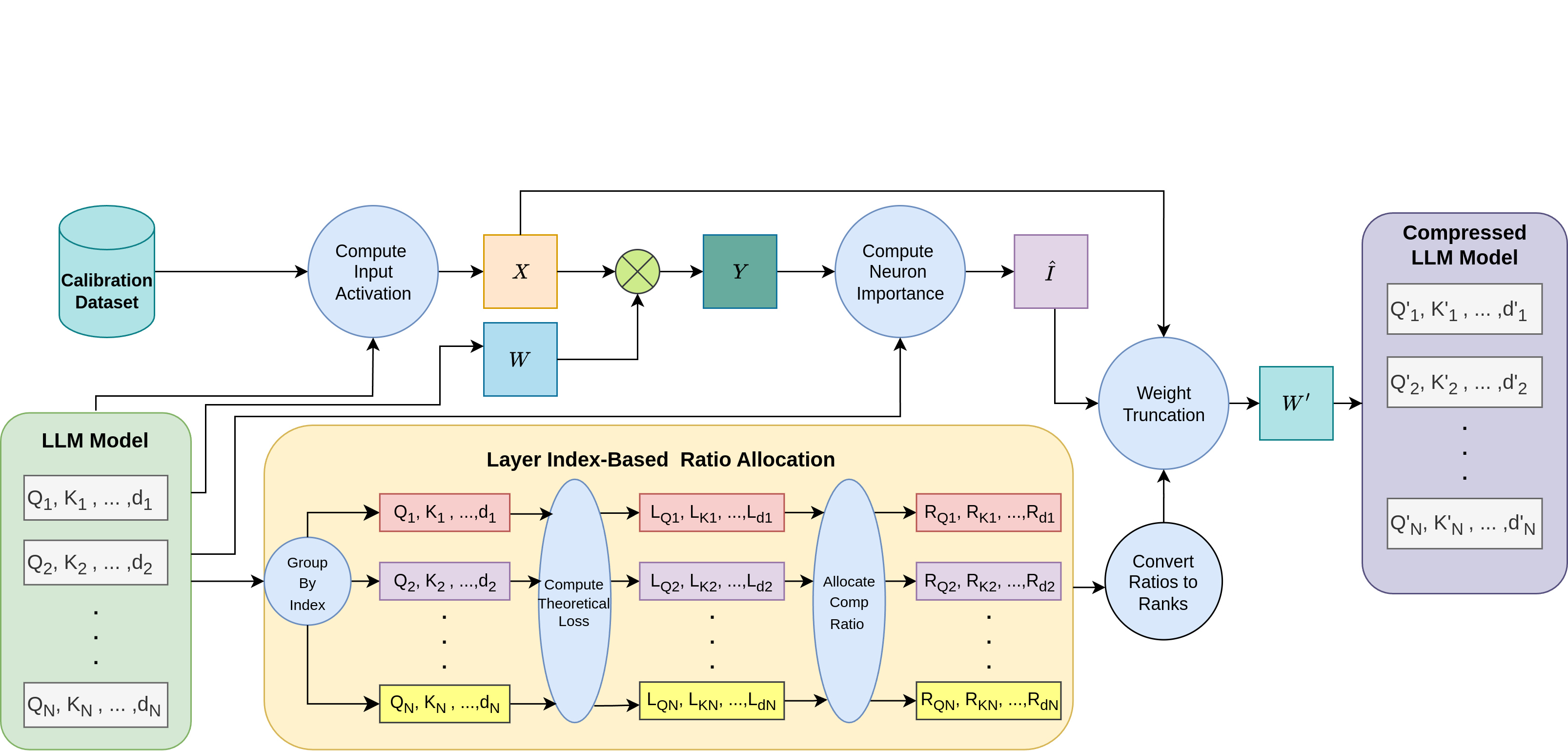} 
    \caption{Architectural overview of the hybrid compression pipeline (NIDA-SVD). Our approach synthesizes the weight importance with data-aware low-rank approximation, guided by our dynamic rank allocation scheme. In the architectural representation, layer components such as $\mQ_{i}$ (Query) and $\mK_{i}$ (Key) are used to denote the corresponding compressible weight matrices $\mW_{Q_i}$ and $\mW_{K_i}$ respectively.} 
    \label{fig:architecture}
\end{figure*}

%% file: sections/related-work.tex
Several approaches have been proposed to compress    transformer architectures \cite{transformer} in pre-trained LLMs \cite{bert,roberta,deberta,debertav3,llama,gemma2,gemma3,gpt}. Broadly, these methods fall into four categories.

\textbf{Pruning} methods \cite{prunning-1,prune-slicegpt,llm-surgeon-prune,sparse-gpt-prune} aim to eliminate as many parameters as possible, effectively zeroing out the corresponding weights in the model’s transformation matrices. On the one hand, \textbf{unstructured} pruning removes individual parameters without constraints, but typically requires specialized hardware for efficient deployment.  On the other hand, to improve hardware compatibility, \textbf{structured} pruning removes entire rows or columns from weight matrices, which aligns well with standard hardware architectures optimized for fast matrix multiplication.

\textbf{Distillation}-based methods \cite{distilbert,tinybert,dynabert} leverage knowledge distillation \cite{distillation-survey} to reduce the number of neural network (NN) parameters by training a smaller model to mimic the behavior of the original. While often effective, this approach can be computationally expensive, as it requires retraining a model from scratch.

Third, and most relevant to our work, \textbf{low-rank approximation} methods \added{\rcite{comp-survey4}} approximate transformer weight matrices using singular value decomposition (SVD), factorizing them into the product of two low-rank matrices. Standard SVD treats all entries as equally important for minimizing reconstruction error. FWSVD \cite{fwsvd} improves on this by weighting parameters according to their importance to the model’s output. Parameter importance is determined via differentiating the task loss with respect to the parameter and taking the magnitude of the gradient, a process that requires additional access to a calibration dataset. A more recent study \cite{features-low-rank-weights-not} revealed that most transformer matrices are not low rank, and thus, when trying to compress them via low-rank approximation, the resulting compressed networks exhibit significant performance drops. However, it was found that the opposite holds for intermediate token representations, which appear to have a low rank structure. Since the rank of a matrix product is less than or equal to the minimum among the ranks of individual matrices participating in the multiplication, DRONE \cite{drone}, ASVD \cite{asvd}, and SVD-LLM \cite{svd-llm} approached compression from the perspective of functional equivalence, minimizing errors in the result of matrix multiplications. However, DRONE required caching intermediate feature activations, resulting in large memory requirements. SVD-LLM alleviated this restriction by relying only on their covariance matrix. At the same time, SVD-LLM ensured faithful estimation of the matrix multiplication error based on the truncated singular values, a limitation of ASVD. Yet, in its initial version, SVD-LLM relied on Cholesky decomposition, which (a) requires positive-definite matrices, a requirement that is not always fulfilled, and (b) can suffer from numerical instabilities during iterative optimization. Both issues were addressed by SVD-LLMv2 \cite{svd-llmv2}, using a two-step SVD algorithm. Low-rank approximation has also been combined with knowledge distillation in \cite{matrix-decomp}.
Our approach combines SVD-LLMv2, the best performing method among the previously discussed state-of-the-art, with ideas inspired from FWSVD.

In low-rank approximation approaches, a fundamental hyperparameter that controls the compression rate - performance tradeoff is the rank of the compressed matrix. Since the network itself has varying sensitivity with respect to the different matrices across layers, distributing the compression rate across matrices for optimal performance typically requires an exhaustive sensitivity analysis, which is practically infeasible. 
Some of the approaches \cite{fwsvd,svd-llm} simply distribute the compression rate evenly across layers, while other approaches \cite{drone,asvd,svd-llmv2} heuristically simplify the exhaustive search in more manageable terms. Our approach is computationally efficient and builds on top of the rank selection algorithm of SVD-LLMv2.

Fourth, and orthogonal to the previous approaches for LLM compression, is weight \textbf{quantization} \cite{spin-quantization,quantization-1,actaware-quantization,bert-quantization}, which typically reduces the memory requirements to perform model inference by quantizing the parameters of the LLM under a fixed budget of binary digits. \added{Recent advances in post-training quantization (PTQ) \rcite{ptq}, demonstrate that high-accuracy LLM quantization can be achieved without retraining. In parallel, quantization-aware training (QAT) techniques \rcite{qat} tailored for LLMs further improve robustness by modeling quantization noise during fine-tuning}. Similar to other low-rank approximation techniques, our approach can additionally benefit from such schemes.

%% file: sections/background.tex
In this section, we provide foundational components to understand our approach. We begin by revisiting core elements of the Transformer architecture \cite{transformer}, followed by a description of BERT \cite{bert}, the Language Model that will become the basic subject of our experiments \added{and an overview of the architectures of DistilBERT, MobileBERT, and TinyBERT, given their role as additional models for our compression method.} We then review SVD \cite{svd} and its application as a post-training compression technique. Finally, we provide an overview of two prominent SVD-based methods, FWSVD \cite{fwsvd} and SVD-LLMv2 \cite{svd-llmv2}, which serve both as comparison baselines and inspiration for the development of the proposed method.

\subsection{The Transformer Architecture}
\label{sec:transformer}
In transformer architectures, a token is the smallest unit of input (such as a word, subword, or character) that the model processes after being mapped to an embedding vector. Let $\mX \in \R^{D\times N}$ denote a matrix of $N$ tokens, each one represented by a D-dimensional vector in the embedding space.

The core of a Transformer is a stack of layers. Each \textbf{transformer layer} or \textbf{block} contains two main components: a self-attention mechanism and a feed-forward network (FFN). The self-attention mechanism uses two matrices, namely the query $\mQ = \mW_Q^T\mX$ and key $\mK=\mW_K^T\mX$ to weigh the importance of each token to the others and create a contextualized representation by linearly combining a matrix of values $\mV = \mW_V^T\mX$. The output of the self-attention mechanism is subsequently transformed by a linear layer with weight matrix $\mW_{a}$. The feed-forward network, also referred to as the multi-layer perceptron (MLP) block, then processes the previous output to learn higher-level features. In the self-attention mechanism, four weight matrices are involved, which are amenable to compression, namely $\mW_Q,\mW_K$, $\mW_V$, and $\mW_{a}$, while the feed-forward block introduces two additional matrices, with the first performing up-projection and the second down-projection. We refer to those matrices as $\mW_{u}$ and $\mW_{d}$, respectively.

These two components of the transformer block are wrapped in a residual connection, and each one is followed by layer normalization, which both are essential for stable training and scaling to the deep architectures that characterize modern LLMs. 

\subsection{Neurons and Pre-Activations}
A neuron in a NN can be understood as a simple function that computes the inner product between an input vector and a weight vector. If the input is an embedding vector $\vx \in \mathbb{R}^D$ corresponding to a token representation, and the neuron's weights are $\vw \in \mathbb{R}^D$, the output is given by $y = \vw^T \vx$, which is a scalar representing the neuron's output, also sometimes referred to as the neuron's \textit{pre-activation} due to often becoming the input to a non-linear activation function such as the rectified linear unit (ReLU) \cite{DLBook}.

At a layer level, many neurons are organized together, and their weights form a weight matrix $\mW \in \mathbb{R}^{D \times D^\prime}$ , where $D^\prime$ corresponds to the number of neurons in the layer. Additionally, a collection of $N$ input token representations may be arranged in a matrix $\mX \in \R^{D \times N}$. Then, the layer's output computation can be expressed compactly as $\mY = \mW^T\mX$, where the output $\mY \in \mathbb{R}^{D^\prime \times N}$ is a matrix, and each element of $\mY$ corresponds to the output of a single neuron for a single token. Typically, a \textbf{transformer layer} is comprised of a set of weight matrices, each one being responsible for supporting a different mechanism, such as the self-attention mechanism or the transformation of the feed-forward network. Still, we will be referring to these individual weight matrices as \textbf{layers} due to stacking a collection of neurons. In most cases, we will be mentioning transformer layers explicitly, but occasionally whether the term layer refers to them can be inferred from context.

\subsection{\replaced{BERT-based Architectures}{BERT}}
The standard Bidirectional Encoder Representations from Transformers (BERT) model is an encoder-only architecture composed of 12 identical Transformer blocks, all following the description given in Section \ref{sec:transformer}. BERT pre-training involves bi-directional mask-language modeling, i.e. training the network to predict masked tokens by conditioning both on the left and on the right context of the masked token, and additionally performing next sentence prediction, i.e. to predict whether a second sentence is a natural continuation of the first one. This pre-training scheme allows BERT to excel in NLU tasks with subsequent fine-tuning and a minimal computation budget. 

\added{DistilBERT is essentially a compact version of BERT, successfully distilled down to 6 layers half the size of the original resulting in $40\%$ fewer parameters. The trick is knowledge distillation: the small student model is trained not just on text, but also on the outputs of the larger BERT teacher, which allows it to retain a remarkable $\approx 97\%$ of the original model's language understanding capability.}

\added{MobileBERT was designed specifically for speed and efficiency on mobile devices. Its unique structure uses a bottleneck design by introducing an intermediate projection layer that significantly narrows the Transformer's hidden dimensions, coupled with factorized self-attention mechanisms. This innovative layer decomposition reduces the parameter count and latency dramatically, allowing it to maintain strong performance while being highly efficient.}

\added{TinyBERT adopts a more aggressive strategy for model compression, aiming to substantially reduce model size while preserving downstream performance. Its training relies on a comprehensive two-stage knowledge distillation framework. In the first stage, the student model learns from the teacher’s internal representations by distilling information from intermediate Transformer layer outputs. In the second stage, it captures the teacher’s behavior by distilling its attention distributions. By jointly aligning both hidden states and attention patterns with those of the teacher, TinyBERT achieves significant compression often up to $7.5\times$ smaller than BERT base while maintaining strong task performance.}

\subsection{Low-rank approximation for model compression}
A common approach to reduce the size of large neural networks is via SVD and low-rank approximation. A weight matrix $\mW \in \R^{m \times n}$ can be approximated as $\mW^\prime = \mU_r\mSigma_r\mV_r^T$, where $\mU_r \in \R^{m \times r}$ and $\mV_r \in \R^{n \times r}$ are truncated orthogonal matrices, $\mSigma_r \in \R^{r \times r}$ is a diagonal matrix of the top $r$ singular values, and $r \le \text{rank}(W)$. The singular values in $\mSigma_r$, ordered from largest to smallest, represent the contribution of the corresponding components of $\mU_r$ and $\mV_r$ in the original matrix $\mW$.

This technique is used to compress the large, weight matrices found in models like BERT. Instead of storing the full original matrix $\mW \in \R^{m \times n}$, we can store two smaller matrices, $\mA \in \R^{m \times r}$ and $\mB \in \R^{r \times n}$, which, when multiplied, approximate the original matrix ($\mW \approx \mA\mB$). This can be done by simply setting $\mA = \mU_r\sqrt{\mSigma_r}$ and $\mB=\sqrt{\mSigma_r}\mV_r^T$, with $\sqrt{\mSigma_r}$ denoting the matrix with all elements equal to the square root of $\mSigma_r$.

According to our definition, the original model's matrix $\mW$ has $mn$ parameters. The total number of parameters in $\mA$ and $\mB$ is $mr + rn = r(m+n)$, where r is the chosen rank for the approximation. For effective compression, the new parameter count must be less than the original. This condition is met when: $mr + rn < mn$. This concludes that $r$ must be chosen such that $r < (mn)/(m+n)$. A smaller $r$ decreases parameter count and speeds up computations, but increases approximation error, which can affect performance.

\added{The foundational method for low-rank approximation within the Transformer architecture is truncated SVD (tSVD), which is deterministic. tSVD works by isolating and keeping only the most important singular components (vectors and values) while discarding the rest. The complexity of tSVD is managed by utilizing specialized Krylov-based algorithms \rcite{tsvd} that focus computation exclusively on these dominant components, avoiding the cost of calculating the full matrix spectrum. An alternative approach is randimized SVD (rSVD) \rcite{rSVD}, which uses random projections to create a smaller, sampled version of the matrix, significantly reducing the computation time. While rSVD offers a notable speed advantage, our methodology relies on the deterministic tSVD because it provides the essential guarantee of optimality (Eckart-Young theorem) and numerical stability. This stability is crucial for zero-shot compression tasks.}

\subsection{FWSVD}
\label{sec:fwsvd}
Fisher-Weighted Singular Value Decomposition (FWSVD) extends low-rank approximation by integrating the concept of parameter importance. While the SVD method assumes that all parameters of a model's weight matrix $\mW$ have the same importance to the matrix reconstruction error, FWSVD leverages the Fisher information matrix to quantify the importance of each parameter in $\mW$ to the task loss, ensuring that the most critical components for model performance are better preserved during the compression process. 

For a single parameter $w_{ij}$ at the location $(i,j)$ in a model's weight matrix $\mW$, the Fisher Information, denoted as $\hat{I}_{w_{ij}}$, measures the amount of information that a dataset $D$ provides about the parameter. This is computed as the sample average of the squared partial derivative of the task's loss function $L(d_i; w_{ij})$ with respect to $w_{ij}$, given by:
\begin{equation}
\label{eq;fisher-pi}
\hat{I}_{w_{ij}} = \frac{1}{|D|} \sum_{i=1}^{|D|} \left( \frac{\partial}{\partial w_{ij}} L(d_i; w_{ij}) \right)^2    
\end{equation}
with $d_i$ denoting the $i$-th sample of the dataset and $|D|$ its total number of samples. 

The Fisher-Weighted approach changes the optimization objective from a generic mathematical one to a task-specific one. While the standard SVD objective is to find a low-rank approximation $\mW^\prime$ that minimizes the reconstruction error, expressed as $\text{min}_{\mW'}||\mW-\mW'||_2$, the FWSVD objective is given by $\text{min}_{\mW'} ||\hat{\mI}_w\circ(\mW-\mW')||_F$ with $\hat{\mI}_w$ a matrix with elements $\hat{I}_{w_{ij}}$ and $\circ$ denoting the Hadamard product. However, this optimization problem does not have a closed-form solution. For this reason, the authors of \cite{fwsvd} proposed an approximation: $\text{min}_{\mW'} ||\hat{\mI}\mW-\hat{\mI}\mW'||_F$, with $\hat{\mI}$ a diagonal matrix whose diagonal element $\hat{I}_i$ in row $i$ is equal to $\hat{I}_i = \sqrt{\sum_j \hat{I}_{w_{ij}}}$. In this way, each neuron $i$ is assigned an importance weight $\hat{I}_i$ based on the Fisher information that the dataset provides regarding its parameters. The approximation holds due to the fact that whenever all elements in a row of $\mW$ share the same importance, the Hadamard product can be written as a standard matrix product with a diagonal matrix.

\subsection{SVD-LLMv2}
\label{sec:background-svd-llmv2}
The most prominent SVD-based work is SVD-LLMv2, which combines the concepts of data-aware low-rank approximation with a computationally efficient and effective rank allocation algorithm, an algorithm that assigns ranks to each one of the approximated matrices under a predefined parameter budget. Instead of minimizing the error on the weight matrix (as SVD and FWSVD do), the optimization objective of SVD-LLMv2 is defined on the layer's output pre-activations, given by $\text{min}_{\mW'}||\mW\mX-\mW'\mX||_2$, with $\mX$ denoting a matrix of token representations to be transformed by the layer's matrix $\mW$. The SVD-LLMv2 approach, by focusing on the layer's functional behavior, has been shown to better preserve task performance. 

Additionally, the algorithm for matrix approximation that was introduced by SVD-LLMv2 ensures that the error in the layer's output pre-activations can be directly predicted from the truncated singular values. This is a significant improvement over previous approaches \cite{asvd,drone}, which suffer from a sharp drop in performance when truncating the smallest singular values due to the fact that they are not directly related to the truncation error of the layer's pre-activations.

Rather than applying a uniform compression ratio to every layer, which can be inefficient given that the sensitivity of the network with respect to different layers may vary, SVD-LLMv2 considers a heterogeneous allocation. It first groups the weights by their type, regardless of transformer layer index (e.g $\mW_K, \mW_Q, \mW_V, \mW_a, \mW_u, \mW_d$, etc.). Then it assigns a target rank to each of the weight matrices in the group by using the error in the layer's pre-activations under uniform rank allocation as a proxy for the layer's sensitivity with respect to rank reduction. This approach considers the matrix sensitivity to rank reduction from a functional perspective, resulting in a better trade-off between model size and performance preservation. What makes this grouping strategy efficient is that matrices within the same group share the same target parameter sub-budget, i.e. a parameter sub-budget is defined for each group based on the overall target compression ratio, and subsequently this sub-budget is distributed across matrices within the group. 

%% file: sections/motivation.tex
\subsection{Combining Neuron Importance with Data-Aware Low Rank Approximation}
\label{sec:motivation-neuron-impartance}
Summarizing the discussion of the previous Section, to compute the low-rank approximation of a weight matrix, on the one hand, FWSVD considers the effect of each parameter on the task loss, which we refer to as \textbf{parameter importance} (PI). On the other hand, SVD-LLMv2 considers the effect of the parameters on the layer's output pre-activations, aiming for a functional equivalence of the layer before and after the compression. In this work, we are inspired by both methods in order to explore an approach that combines ideas from both perspectives. In particular, instead of minimizing $\min_{\mW'}||\mW\mX - \mW'\mX||_F$ aiming for low reconstruction error on the layer's output, i.e. the functional equivalence objective of SVD-LLMv2, we also consider the effect of the layer's output pre-activation on the task loss. This leads to the following optimization objective $\min_{\mW'}||\hat{\mI}_y \circ (\mW\mX-\mW'\mX)||_F$, with $\hat{\mI}_y$ denoting a matrix containing the importance of each neuron's output to the task loss. 

Solving this optimization problem faces the same limitation as the original formulation of FWSVD, i.e. a lack of closed-form solution. For this reason, we also approximate the Hadamard product by considering standard multiplication with a diagonal matrix $\hat{\mI}$ whose diagonal element in row $i$ is equal to the overall importance of the $i$-th neuron to the task loss. The latter is computed by aggregating the neuron's importance over several samples in a calibration dataset. Thus, in our approach, we find $\mW'$ that minimizes $||\hat{\mI}\mW\mX - \hat{\mI}\mW'\mX||_F$. A first intuitive choice for computing the diagonal elements of $\hat{\mI}$ is to aggregate \textbf{parameter importances} for each neuron, as it is done in FWSVD. However, in the experiments, we provide statistically significant evidence that this is suboptimal, leading to compressed networks that perform worse than when using data-aware low-rank approximation alone. As we discuss in Section \ref{sec:neuron-importance}, we were able to improve on that by considering a different strategy. This strategy is based on a \textbf{direct way} to measure \textbf{neuron importance} (NI) that does not rely on the \textbf{parameter importance} previously considered by FWSVD.

\subsection{Dynamic Rank Allocation Across Matrices}
\label{sec:motivation-rank-allocation}
Each target compression ratio ($CR$) can be directly translated to a total parameter budget under which the compressed LLM should fit. In low-rank approximation, the number of parameters in a layer is directly controlled by the rank of the approximation. Thus, distributing the total available parameters to individual layers can be accomplished by deciding on the rank of the approximation for each weight matrix. Performing \textbf{rank allocation}, that is, distributing target matrix ranks under a fixed parameter budget to individual layers, is a combinatorial and practically intractable problem, as optimal allocation relies on exhaustive search. 

Previous approaches either resorted to proportionally equivalent reduction of parameters across layers \cite{fwsvd,svd-llm} (which we call uniform allocation), or based their decision on algorithms that were inspired by studies in a reduced search space. For instance, \cite{asvd} found that optimal rank allocation varies both with layer depth and the layer's functional type (i.e. if it corresponds to $\mW_Q,\mW_V,\mW_a,\mW_u$, etc.). Meanwhile, \cite{drone} found that compression distortions from lower layers may result in progressively accumulative errors towards the latter layers, implying that rank allocation should vary with depth and compressing lower layers should be more conservative compared to higher ones. Additionally, \cite{svd-llmv2} focused on allocating ranks for layers in the same functional group, exploiting the fact that the matrices in the same group share the same size. Despite the reduction in the search space, the rank allocation algorithm of \cite{drone} is still computationally expensive and is driven by model performance instead of parameter budget. Even though the rank allocation algorithm of \cite{svd-llmv2} is computationally efficient, adapting it to a different grouping strategy requires addressing the challenge of matrix size variability among matrices in the same group.  Our approach addresses these challenges by building on the findings of \cite{drone}, while adapting the computationally efficient algorithm of \cite{svd-llmv2} to group layers by their depth index instead of their functional role.


%% file: sections/method.tex
Our proposed hybrid compression methodology is a post-training, SVD-based pipeline that fuses key concepts from recent advancements to achieve more efficient and effective compression. As shown in Figure 1, our approach synthesizes data-aware low-rank approximation with the concept of Fisher-weighted neuron importance. We also introduce a heterogeneous rank allocation that improves upon previous methods by adapting compression on a per-layer basis. The following subsections detail the core components of our pipeline, outlining our complete process from the computation of neuron importance to the final data-aware low-rank approximation and rank allocation algorithm.

\subsection{Neuron Importance Estimation}
\label{sec:neuron-importance}
We propose to estimate neuron importance by computing the amount of Fisher Information that a dataset provides about the neuron's \textbf{output pre-activation}, \textbf{instead of the neuron's parameters}. Following the notation of Section \ref{sec:motivation-neuron-impartance}, we calculate the element of $\hat{\mI}_y$ at the position $(i,j)$ for the $k$-th sample of the dataset $d_k$ as:
\begin{equation}
    \hat{I}_{y_{ij}} = \left( \frac{\partial}{\partial y_{ij}} L(y_{ij}(d_k)) \right)^2
\end{equation}
with $L$ the task loss function and
 $y_{ij}$ a spatial element in the matrix of layer outputs: $\mY=\mW^T\mX$. We subsequently form the neuron importance matrix $\hat{\mI}$, a diagonal matrix with its element in row $i$ being equal to the importance of neuron $i$:
\begin{equation}
    \hat{I}_i = \sqrt{\frac{1}{|D|} \sum_{k=1}^{|D|}\E_j\big[\hat{I}_{y_{ij}} | d_k\big]}
\end{equation}

\subsection{Data-Aware Low Rank Approximation driven by Neuron Importance}
We adapt the weight truncation algorithm of SVD-LLMv2 by integrating the neuron importance matrix $\hat{\mI}$ of the previous section into the process of low-rank approximation. Algorithm \ref{algo:weight_truncation} summarizes the process of decomposing a matrix $\mW$ to its low-rank approximation $\mW'=\mA\mB$ and is a generalization over the algorithm proposed by SVD-LLMv2, with the two algorithms being identical whenever the neuron importance matrix is equal to the identity matrix.
\input{sections/algorithm/truncation}
\newtheorem{theorem}{Theorem}

\added{
Let $L=||\hat{\mI}\mW\mX - \hat{\mI}\mW^\prime\mX||_F$ denote the weighted compression loss when approximating $\mW$ with a low rank matrix $\mW^\prime$. For any given rank of the approximation, the theoretical minimum loss is given by the error induced by the truncated SVD decomposition of $\hat{\mI}\mW\mX$, denoted as $||\mathbf{SVD}(\hat{\mI}\mW\mX)||_F$.
}
\begin{theorem}
\label{theorem-truncation}
\added{
If $\mU_s, \mS_s, \mV_s$ are obtained by the SVD decomposition of $\mX\mX^T$ and $\mU_{ws}^t,\mS_{ws}^t,\mV_{ws}^t$ are obtained by the truncated SVD decomposition of $\hat{\mI}\mW\mU_s\sqrt\mS_s$, the compressed weight matrix $\mW^\prime = \hat{\mI}^{-1}\mU_{ws}^t\mS_{ws}^t\mV_{ws}^t\sqrt{\mS_s}^{-1}\mU_s^{-1}$ minimizes the weighted compression loss $L$.
}
\end{theorem}
\begin{proof}
\added{Let $\mU_x,\mS_x,\mV_x$ denote the matrices from the SVD decomposition of $\mX$. Since $\mX\mX^T$ is symmetric, $\mU_s=\mV_s$. Moroever, $\mU_x=\mU_s$ and $\mS_x=\sqrt{\mS_s}$. Let also $\mC=\mU_s\sqrt{\mS_s}$ which implies that $\mC^{-1}=\sqrt{\mS_s}^{-1}\mU_s^{-1}$. It is easy to show that $\mC^{-1}\mX = \mV_x$ which is orthogonal and thus does not affect the Frobenius norm under matrix multiplication. For the proposed choice of $\mW^\prime$ the weighted compression loss $L$ becomes}:
\begin{align*}
    L &= ||\hat{\mI}\mW\mX - \hat{\mI}\mW^\prime\mX||_F  \\
      &= ||\hat{\mI}\mW\mX - \hat{\mI}\hat{\mI}^{-1}\mU_{ws} ^t\mS_{ws}^t\mV_{ws}^t\sqrt{\mS_s}^{-1}\mU_s^{-1}\mX||_F \\
      &= ||\hat{\mI}\mW\mC\mC^{-1}\mX - \mU_{ws} ^t\mS_{ws}^t\mV_{ws}^t\mC^{-1}\mX||_F \\ 
      &= ||(\hat{\mI}\mW\mC - \mU_{ws} ^t\mS_{ws}^t\mV_{ws}^t)\mC^{-1}\mX||_F\\
      &= ||(\hat{\mI}\mW\mC - \mU_{ws} ^t\mS_{ws}^t\mV_{ws}^t)||_F \\
      &= ||\mathbf{SVD}(\hat{\mI}\mW\mC)||_F \\
      &= ||\mathbf{SVD}(\hat{\mI}\mW\mU_s\sqrt{\mS_s})||_F \\
      &= ||\mathbf{SVD}(\hat{\mI}\mW\mU_x\mS_x)||_F \\
      &= ||\mathbf{SVD}(\hat{\mI}\mW\mU_x\mS_x\mV_x)||_F \\
      &= ||\mathbf{SVD}(\hat{\mI}\mW\mX)||_F
\end{align*}
\added{which is equal to the theoretical minimum.}
\end{proof}
\added{Algorithm \mbox{\ref{algo:weight_truncation}} computes $\mW^{\prime}$ based on Theorem \mbox{\ref{theorem-truncation}} and suggests a low rank matrix approximation that minimizes the weighted compression loss.}
\subsection{Dynamic Rank Allocation}
\label{sec:rank-allocation}

Following Section \ref{sec:background-svd-llmv2}, and based on the findings of \cite{drone} we propose an adaptation of the rank allocation algorithm of \cite{svd-llmv2} by changing the grouping strategy of the weight matrices to be based on transformer layer index (i.e. layer's depth in the transformer's sequence of layers). This change in the grouping implies that a target parameter sub-budget is defined for all matrices within a transformer layer, instead of matrices with the same functional role. The algorithm still distributes this sub-budget to each matrix in the group based on matrix sensitivity to rank reduction, as it was proposed in \cite{svd-llmv2}.

Our layer index-based rank allocation method groups all weight types (e.g. $\mW_Q$, $\mW_K$, $\mW_V$, $\mW_a$, $\mW_u$ and $\mW_d$) with the same transformer layer index $l$, by treating each one of the transformer layers in the BERT-family models as a distinct group. This layer-wise approach modifies the algorithm to derive a compression ratio $r$ for each weight matrix and ensures iterative refinement for optimal balance.
\input{sections/algorithm/allocation}

Algorithm \ref{algo:ratio_allocation} summarizes our layer index-based rank allocation strategy. The algorithm requires as an input the original LLM, a representative set of input activations resulting from the calibration dataset, and the effective compression ratio. The effective compression ratio is calculated to account for parameters non amenable to compression, possibly due to the high sensitivity of the network with respect to these parameters, while still aiming for a specific overall target compression ratio. The algorithm first groups the model's weights by their layer (line 2) and then computes a normalized theoretical loss (lines 6-10), which serves as a proxy for that layer's sensitivity to rank reduction. For a given matrix $\mW$, the theoretical loss function first translates the target compression ratio ($R$) into the corresponding retained SVD rank ($k$) to compute its low-rank approximation $\mW'$. It then uses $\mW'$ to calculate the error on the layer's output pre-activations $L_{\min} = \|\mW\mX - \mW'\mX\|_{F}$, indicated as \textbf{theoretical loss}. Compression ratios are then distributed proportionally to these normalized errors (lines 12-15), ensuring that layers deemed more sensitive to rank reduction receive higher ranks. Line 18 performs the necessary transformation from the calculated proportional ratios ($R_d$) into the specific, non-zero integer SVD ranks ($k$) required for each weight matrix $\mW$ in the model ($M$). To further meet the consistency with the global compression target, the iterative refinement loop (lines 23-34) adjusts matrix ranks to ensure that the desired ratio is achieved. The update rule (lines 28-29) ensures convergence by increasing/decreasing matrix ranks whenever the compressed model is under/over the target parameter budget. This dynamic adjustment makes the allocation procedure adaptive and computationally efficient, avoiding exhaustive search strategies.

Overall, our layer index-based rank allocation algorithm balances two objectives: (i) preserving the most critical layers through proportional loss-aware rank allocation, and (ii) meeting the global compression constraint through iterative refinement. 

\input{sections/figures/weights-comparison}
\input{sections/figures/ratio-comparison}

%% file: sections/algorithm/truncation.tex
\begin{algorithm}[t]
\captionsetup{font=small}
\caption{Weight Truncation Algorithm}
\small
\begin{algorithmic}[1] 
\Statex \textbf{Input:} $\mW$: Original weight matrix 
\Statex \hspace{3em} $\mX$: Matrix of layer input activations
\Statex \hspace{3em} $\hat{\mI}$: Diagonal matrix of neuron importances
\Statex \hspace{3em} $R$: Target rank for the low-rank approximation
\Statex \textbf{Output:} $\mA,\mB,\mW'$: Low rank approximation factors $\mA,\mB$ and compressed weight matrix $\mW'$.
\Procedure{Weight\_Truncation}{$\mW, \mX, \mI, R$}
    \State $\mS \gets \mX\mX^T$  \Comment{Construct matrix $S$ from $X$}
    \State $\mU_s, \mS_s, \mV_s \gets \boldsymbol{\textrm{SVD}}(\mS)$ \Comment{Perform SVD on matrix $S$} 
    \State $\mD \gets \hat{\mI} \mW \mU_s \sqrt{\mS_s}$  \Comment{Construct matrix $D$}
    \State $\mU_{ws}, \mS_{ws}, \mV_{ws} \gets \boldsymbol{\textrm{SVD}}(\mD)$ \Comment{Perform SVD on matrix $\mD$}
    \State $\mU_{ws}^t, \mS_{ws}^t, \mV_{ws}^t \gets \boldsymbol{\textrm{Truncate}}(\mU_{ws}, \mS_{ws}, \mV_{ws}, R)$ \Comment{Perform SVD truncation based on target rank approximation $R$}
    \State $\mA \gets \hat{\mI}^{-1}\mU_{ws}^t\sqrt{\mS_{ws}^t}$ \Comment{The first matrix of low-rank approximation}
    \State $\mB \gets \sqrt{\mS_{ws}^t}\mV_{ws}^t\sqrt{\mS_s}^{-1}\mU_s^{-1}$ \Comment{The second matrix of low-rank approximation}
    \State $\mW' \gets \mA\mB$ \Comment{Reconstructed low-rank approximation $\mW'$}
    \State \Return{$\mA,\mB,\mW'$}
\EndProcedure
\end{algorithmic}
\label{algo:weight_truncation}
\end{algorithm}

%% file: sections/algorithm/allocation.tex
\begin{algorithm}[t]
\captionsetup{font=small}
\caption{Layer Index-based Rank Allocation Algorithm }
\small
\begin{algorithmic}[1] 
\Statex \textbf{Input:} $M$: Original LLM
\Statex \hspace{3em} $\mX$: Input activations
\Statex \hspace{3em} $R$: Effective target compression ratio
\Statex \textbf{Output:} $R_d$: Compression rank allocation per layer (list of lists) 
\Procedure{Rank\_Allocation}{$M, \mX, R$}
    \State $G \gets \boldsymbol{\textrm{Group}}(M)$ \Comment{Group weights by layer index}
    \State $R_d \gets \emptyset$ \Comment{Initialize the compression ratio list} 
    \For{$g$ \textbf{in} $G$}
        \State $L_G \gets \emptyset$ \Comment{Initialize the loss list in the group} 
        \For{$\mW$ \textbf{in} $g$}
            \State $L_{min} \gets \boldsymbol{\textrm{Theoretical\_Loss}}(\mW, \mX, R)$    \Comment{Compute per weight theoretical loss}
            \State $L_{G} \gets L_{G} \cup L_{min}$ \Comment{Append loss to list}
        \EndFor
        \State $L_G \gets 1 / \boldsymbol{\textrm{Log}}(L_G)$    \Comment{Normalize $L_G$}
        \State $layer\_ratios \gets \emptyset$ \Comment{Initialize layer compression ratios list} 
        \For{$L_{min}$ \textbf{in} $L_G$}
            \State $r \gets \boldsymbol{\textrm{Len}}(L_{G}) \times R\times L_{min} / \boldsymbol{\textrm{Sum}}(L_{G})$ \Comment{Allocate ratios proportionally to normalized loss}
            \State $layer\_ratios \gets layer\_ratios \cup r$ \Comment{Append ratio to list}
        \EndFor
         \State $R_d \gets R_d \cup layer\_ratios$  \Comment{Append $layer\_ratios$ list to $Rd$ list}
    \EndFor
    \State $R_d \gets \boldsymbol{\textrm{Convert\_Ratios\_to\_Ranks}}(M,R_d)$ \Comment{Convert ratios to ranks for each weight}
    \State $total\_params \gets \boldsymbol{\textrm{Count\_Total\_Parameters}}(M)$ \Comment{Count all the parameters of the model}
    \State $comp\_params \gets \boldsymbol{\textrm{Count\_Compressed\_Parameters}}(M,R_d)$ \Comment{Count all the parameters of the compressed model}
    \State $achieved\_ratio \gets 1 - comp\_params/total\_params)$    \Comment{Compute achieved ratio after rank allocation}
    \State $i \gets 0$ 
    \While{$|achieved\_ratio - R| > \varepsilon$}
        \State $sign \gets achieved\_ratio - R$ 
        \State $scale\gets (1-R)/(1-achieved\_ratio)$ \Comment{Compute scaling factor}
        \For{$layer\_ranks$ \textbf{in} $R_d$}
            \For{$r$ \textbf{in} $layer\_ranks$}
            \State $r \gets r + sign \cdot i$    \Comment{Iteratively adjust rank based on error direction}
            \State $R_d \gets \boldsymbol{\textrm{Update}}(R_d,r)$    \Comment{Update rank allocation list }
            \EndFor
        \EndFor
        \State $i\gets i + 1$     
    \EndWhile
    \State \Return{$R_d$}
\EndProcedure
\end{algorithmic}
\label{algo:ratio_allocation}
\end{algorithm}

%% file: sections/figures/weights-comparison.tex
\begin{figure*}[!ht]
    \centering
    \parbox{0.45\textwidth}{\centering
        \includegraphics[width=\linewidth]{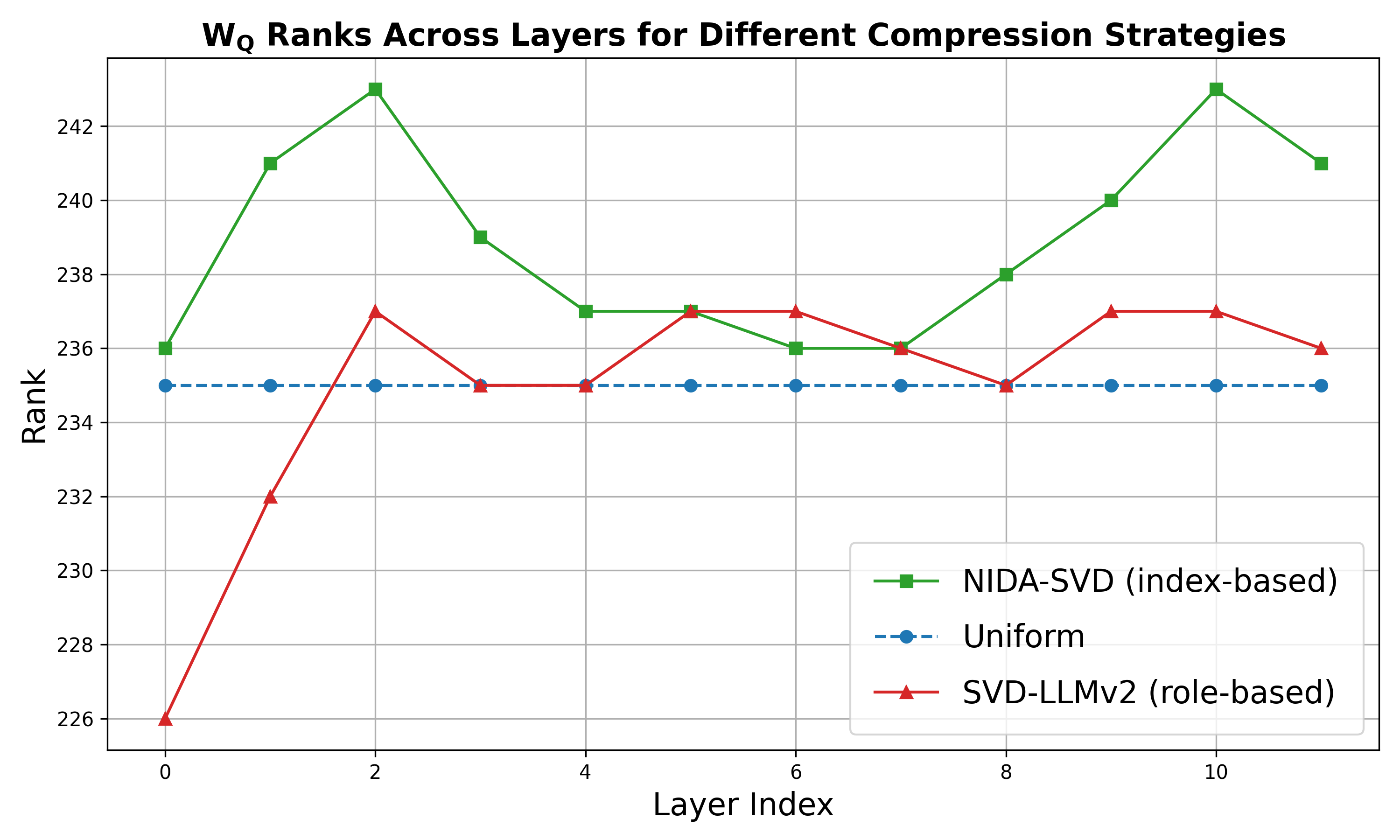}\\
        (a) $\mW_Q$
    }\hfill
    \parbox{0.45\textwidth}{\centering
        \includegraphics[width=\linewidth]{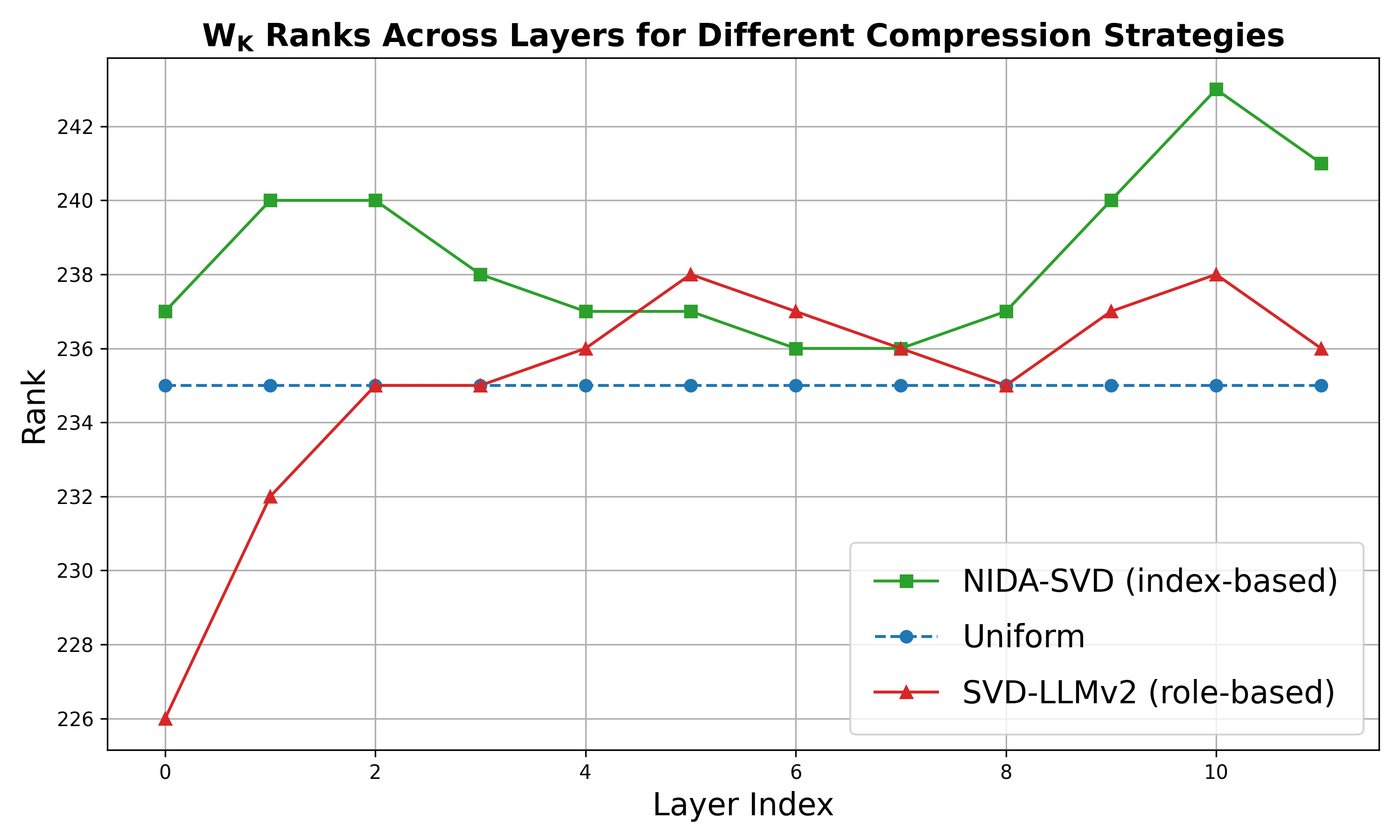}\\
        (b) $\mW_K$
    }

    \parbox{0.45\textwidth}{\centering
        \includegraphics[width=\linewidth]{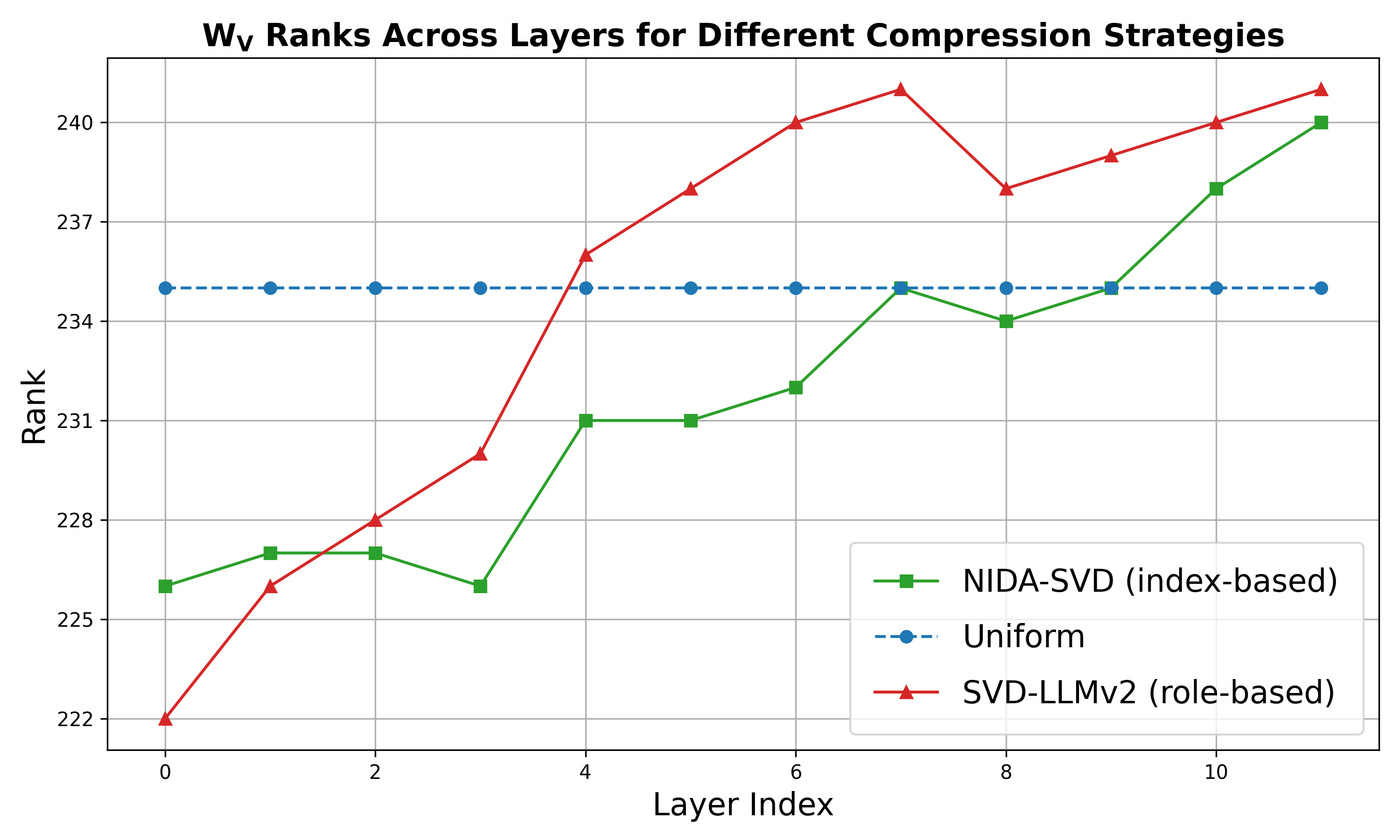}\\
        (c) $\mW_V$
    }\hfill
    \parbox{0.45\textwidth}{\centering
        \includegraphics[width=\linewidth]{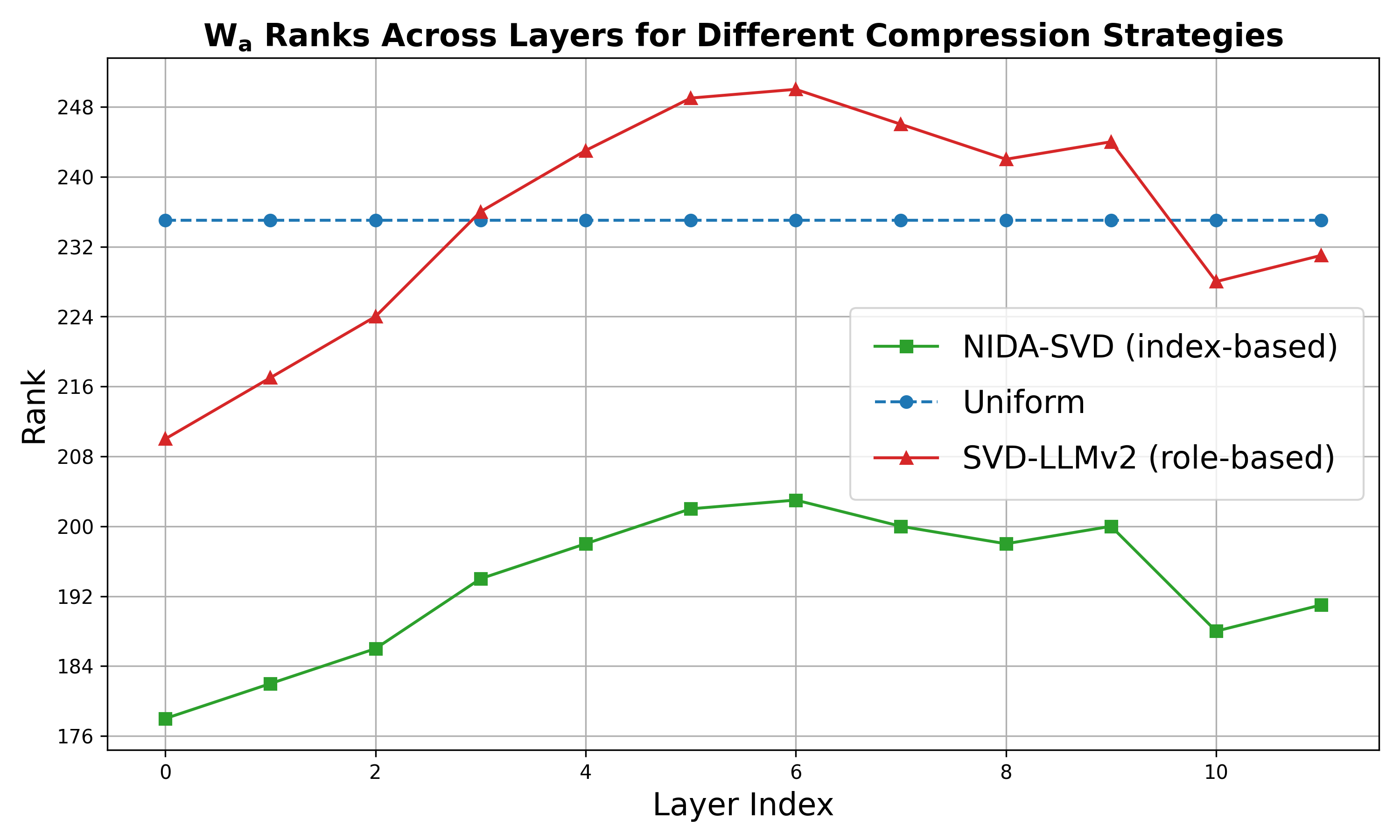}\\
        (d) $\mW_a$
    }

    \parbox{0.45\textwidth}{\centering
        \includegraphics[width=\linewidth]{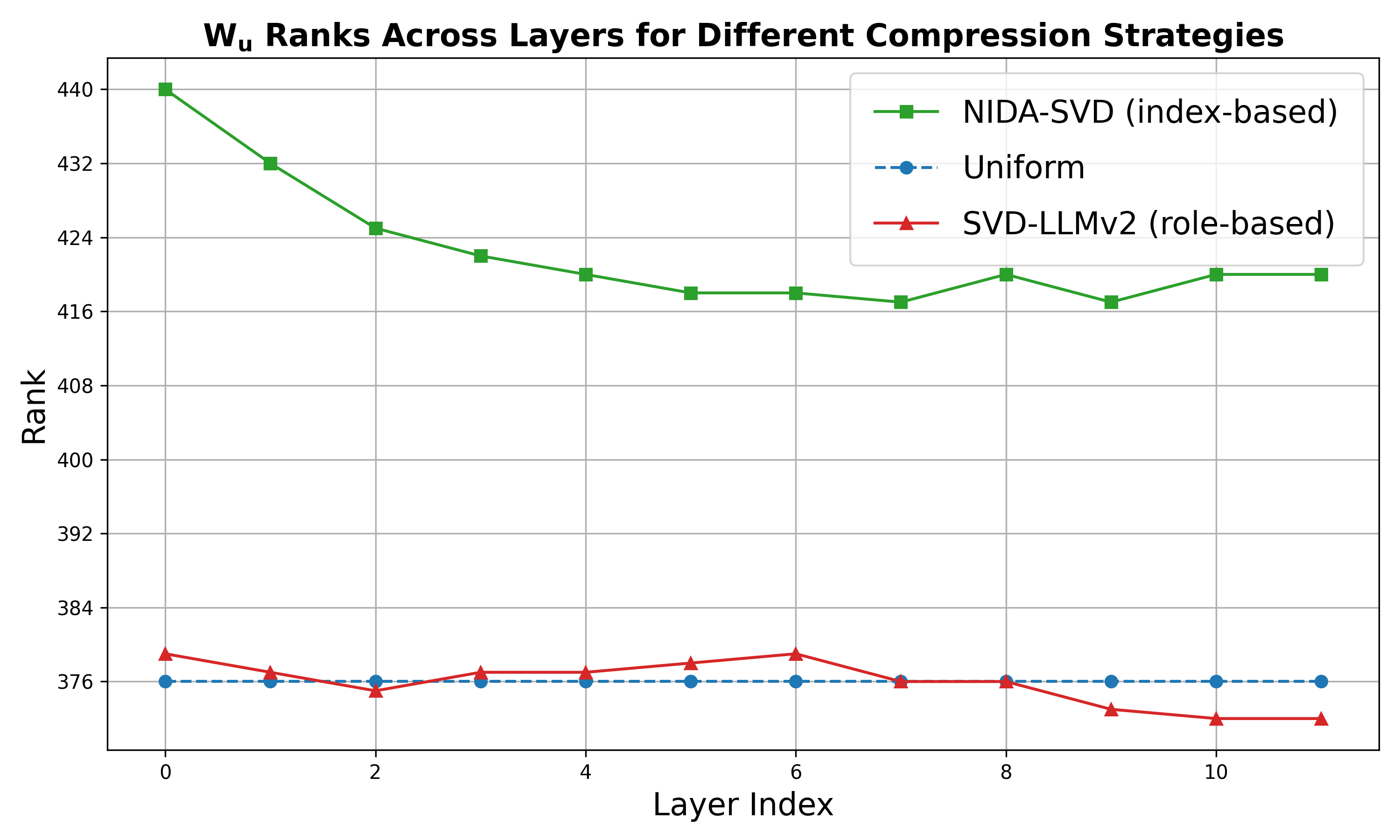}\\
        (e) $\mW_u$
    }\hfill
    \parbox{0.45\textwidth}{\centering
        \includegraphics[width=\linewidth]{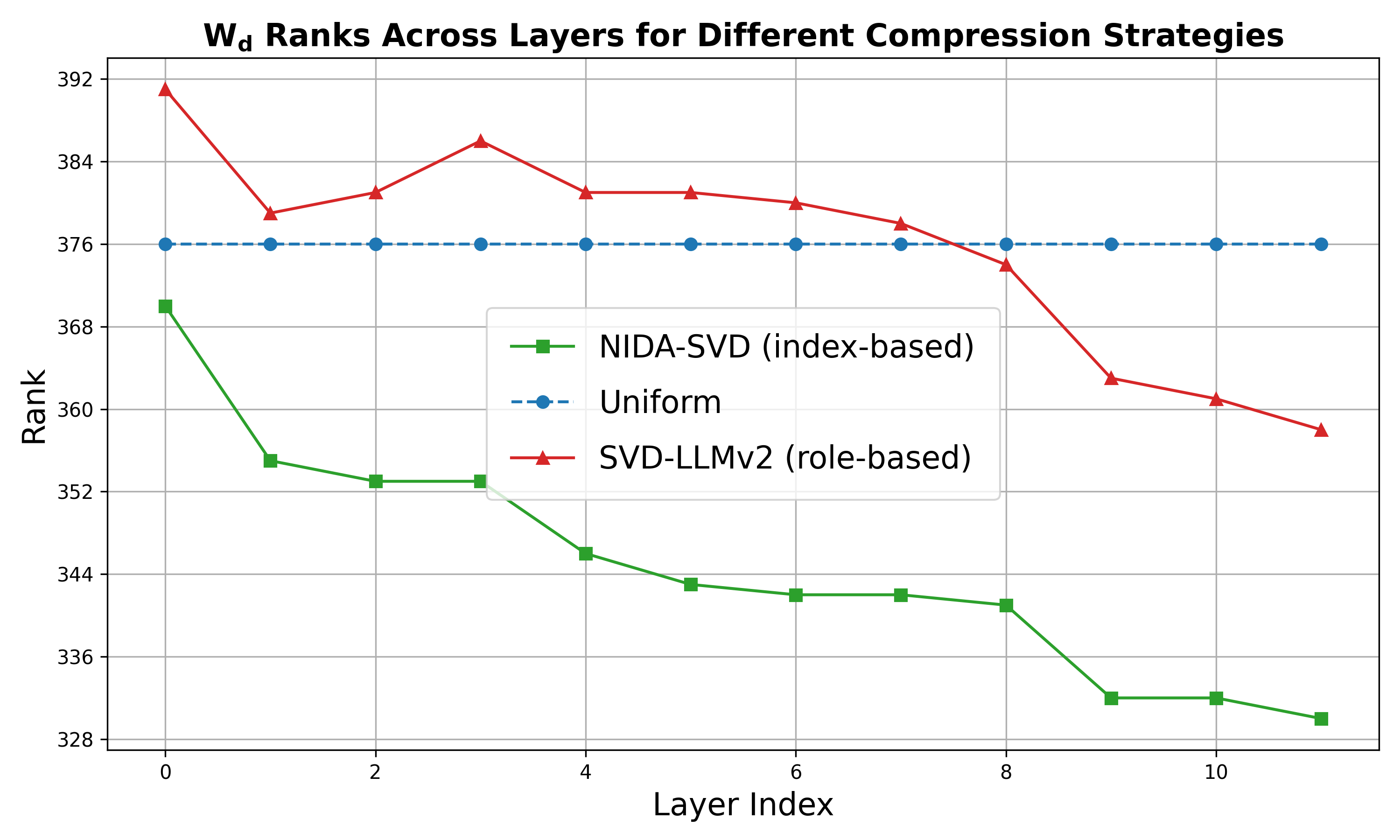}\\
        (f) $\mW_d$
    }

    \caption{Visualization of BERT model ranks under 30\% compression on QNLI dataset, comparing SVD-LLMv2 (role-based), Uniform, and NIDA-SVD strategies: (a) Query, (b) Key, (c) Value, (d) Attention, (e) Up, (f) Down.}
    \label{fig:weights-comparison}
\end{figure*}

%% file: sections/figures/ratio-comparison.tex
\begin{figure*}[!ht]
    \centering
    \parbox{0.45\textwidth}{\centering
        \includegraphics[width=\linewidth]{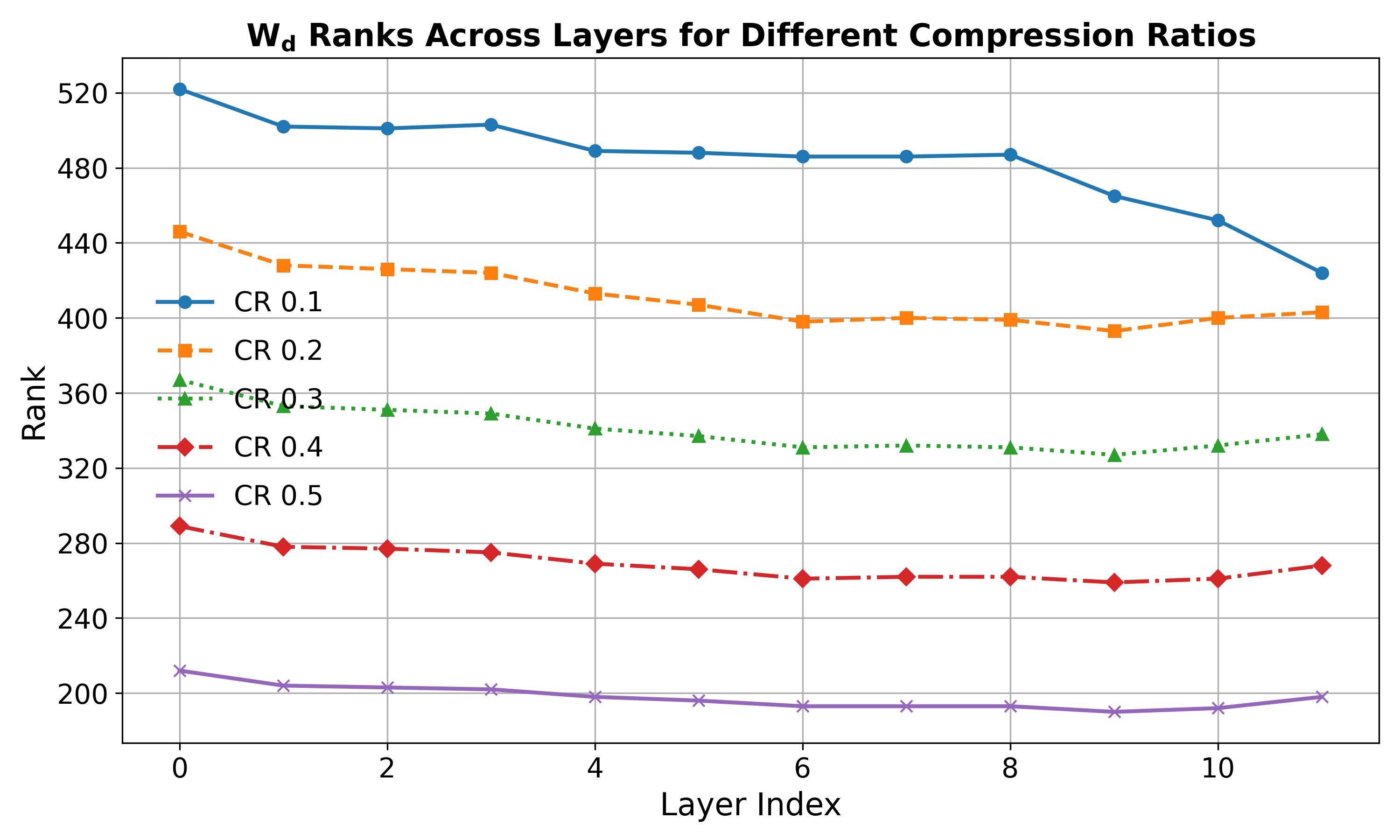}\\
        (a) $\mW_d$
    }\hfill
    \parbox{0.45\textwidth}{\centering
        \includegraphics[width=\linewidth]{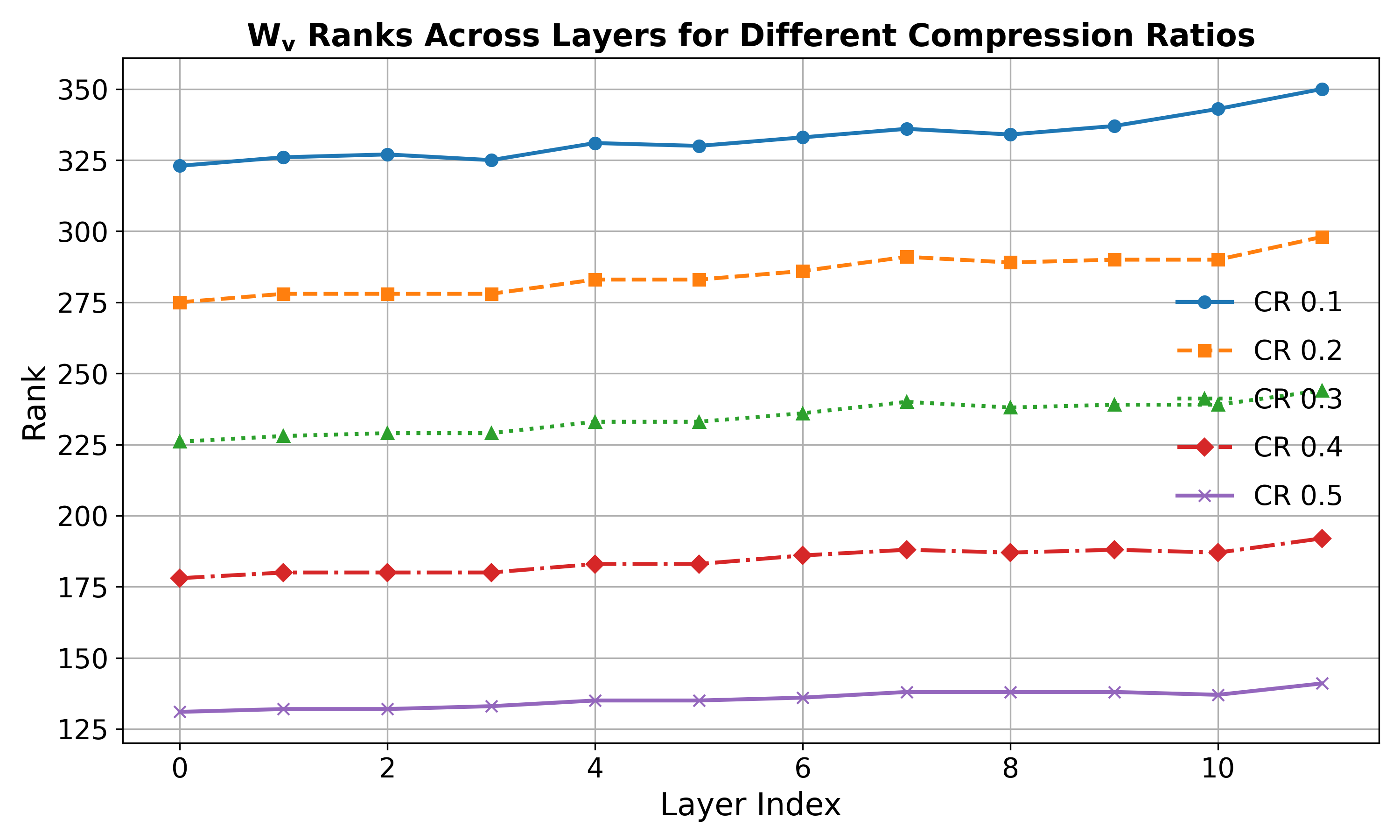}\\
        (b) $\mW_V$
    }
    \caption{Visualization of BERT model weight ranks across different compression ratios (layer-based strategy) on the MRPC dataset: (a) Down, (b) Value.}
    \label{fig:ratio-comparison}
\end{figure*}

%% file: sections/experiments.tex
\subsection{Experimental Setup}
We evaluate the effectiveness of our method by compressing the standard 12-layer BERT transformer architecture. We consider BERT, fine-tuned on 8 datasets of the GLUE-Benchmark \cite{glue-benchmark}, namely \textit{cola, mnli-m, mnli-mm, mrpc, qnli, qqp, sst-2} and \textit{sts-b}. Each dataset corresponds to a different task, and the performance of the (compressed) models is evaluated under a task-specific metric. Single sentence tasks, \textit{cola} and \textit{sst-2} are measured by Matthew’s correlation and classification accuracy, respectively. The sentence similarity tasks \textit{mrpc} and \textit{qqp} are assessed by F-1 score, while \textit{sts-b}, using Pearson-Spearman correlation. Finally, natural language inference tasks \textit{mnli-m}, \textit{mnli-mm} and \textit{qnli} are measured by accuracy.

For each one of the layers in the BERT architecture, the matrices that we consider for compression are the ones discussed in Section \ref{sec:transformer}, namely $\mW_K, \mW_Q, \mW_V, \mW_a, \mW_u, \mW_d$. We do not consider compressing the token embedding matrix, nor the matrices involved in the layer-normalization blocks. To compute the uncentered covariance matrix $\mS$ in Algorithm \ref{algo:weight_truncation}, we randomly sample a (balanced) dataset of 256 text samples from the dataset's training split. We also compute neuron importances (NIs) using the full training split of the dataset, as it is done in FWSVD for parameter importances (PIs). In contrast to FWSVD or SVD-LLMv2 which additionally consider fine-tuning the compressed models to recover performance, we \replaced{emphasize}{report} results and comparisons without a fine-tuning step. Even though fine-tuning may boost the results of any previous SVD-based method, including ours, we want to highlight the benefits of the proposed approach, namely Neuron Importance driven Data-Aware SVD (NIDA-SVD), under a low computation budget which does not involve finetuning. \added{Nevertheless, we report results with fine-tuning when compressing MobileBERT and TinyBERT.}

\added{The experimental setup is extended to include the structurally similar DistilBERT, as well as the specialized MobileBERT and TinyBERT architectures. Since these models are direct descendants of BERT, they are also fine-tuned and evaluated on specific subsets of the GLUE Benchmark using the corresponding task-specific metrics described above. Specifically, the models are evaluated on the following subsets: DistilBERT on \textit{mnli-mm}, \textit{mrpc}, \textit{qnli}, \textit{qqp}, and \textit{sst2}, MobileBERT on \textit{mrpc}, \textit{qnli}, \textit{qqp}, \textit{sst2}, and \textit{stsb}, and TinyBERT on \textit{mnli-mm}, \textit{mrpc}, \textit{qnli}, \textit{qqp}, \textit{sst2}, and \textit{stsb}. For all models, compression is applied to their Transformer encoder blocks and dense layers. The methodology for computing the matrix $\mS$ and the neuron importance $\mI$ remains identical for all models.}

In the following subsections, we present experimental results comparing our method with the recent state-of-the-art in low-rank approximation: SVD, FWSVD, and SVD-LLMv2, and conduct ablation studies under different strategies regarding parameter importance, neuron importance, and rank allocation. Since the algorithm of SVD-LLMv2 performs rank allocation based on the layer's functional role in the transformer block, we call it \textit{role-based allocation}. Contrariwise, our algorithm approaches rank allocation from the perspective of the layer's depth index, and thus we refer to it as \textit{index-based allocation}.

\subsection{Comparison against the state-of-the-art}

Table \ref{tab:soa-uniform} summarizes performance metrics for the BERT model compressed under different Compression Ratios (CR) with a uniform rank allocation strategy. Overall, in 35 out of 40 cases (87.5\%), our method is ranked first; the compressed model retains most of its task performance when compressed with our method.  In the rest of the cases, our method is ranked second, remaining competent to the top-performing method, which is always SVD-LLMv2. We emphasize the improvement that our method offers compared to the previous state-of-the-art, especially in the highest compression rate (0.5): +3.15\% \textbf{absolute improvement} in \textit{cola}, +4.16\% in \textit{mnli-m}, +4.82\% in \textit{mnli-mm}, +1.54\% in \textit{mrpc}, +2.48\% in \textit{qnli}, +2.35\% in \textit{qqp}, and +2.33\% in \textit{stsb}.

\input{sections/tables/soa-uniform}

Similarly, Table \ref{tab:vs-svd-llm-v2} summarizes performance metrics for the BERT model compressed either with our method or the previous state-of-the-art, SVD-LLMv2. In this comparison, we consider each method to use its own rank allocation algorithm: \textit{role-based} allocation for SVD-LLMv2 and \textit{index-based} allocation for the proposed method. For the extended experiments presented below, Table 3 (DistilBERT), Table 4 (MobileBERT), and Table 5 (TinyBERT), each utilized its own rank allocation strategy (role-based for SVD-LLMv2 and index-based for NIDA-SVD). For BERT, In 35 out of 40 cases (87.5\%), our method outperforms SVD-LLMv2 while remaining comparable to SVD-LLMv2 in the remaining 5. In this case, under the highest compression ratio (0.5), our method makes an \textbf{absolute improvement} of +4.41\% in \textit{cola}, +4.94\% in \textit{mnli-m}, +5.01\% in \textit{mnli-mm}, +1.82\% in \textit{mrpc}, +4.14\% in \textit{qnli}, +1.92\% in \textit{qqp} and +6.41\% in \textit{stsb}.

\input{sections/tables/vs-svdllm-v2}

\added{Table}
\ref{tab:distilbert-comp-compact}
\added{presents the DistilBERT results, comparing our NIDA-SVD  against the SVD-LLMv2. Our method shows a clear performance advantage, achieving the best result in 23 out of 25 comparisons (92\%). This advantage is especially pronounced at the compression ratio of 0.3, where NIDA-SVD delivers significant absolute improvements, including +0.90\% in \textit{mnli-mm} and +2.43\% in \textit{qnli}.}

\input{sections/tables/distilbert}

\added{The MobileBERT results, detailed in Table}
\ref{tab:mobilebert-comp-compact}
\added{, compare our NIDA-SVD approach against the SVD-LLMv2. Our methodology demonstrates a clear performance advantage in 21 of the 25 total comparisons (84\%). This general advantage is mostly observed across 0.3 compression ratio, including \textit{qqp} (+3.68\%), \textit{stsb} (+2.03\%), and \textit{sst2} (+1.03\%)}.

\input{sections/tables/mobilebert}

\added{The compression performance on the highly reduced TinyBERT architecture is documented in Table} 
\ref{tab:tinybert-comp-compact}
\added{. Comparing NIDA-SVD against the SVD-LLMv2, NIDA-SVD exhibits a robust performance advantage, securing the superior metric in 17 out of 18 total comparisons (94.4\%). The most substantial performance gains are concentrated at the compression ratio of 0.2. At this level, NIDA-SVD registers notable absolute increases, specifically +1.43\% in \textit{qnli}, +0.98\% in \textit{mnli-mm}, and +0.54\% in \textit{qqp}, confirming the efficacy of the data-aware approach even on extremely compact models.  Note that TinyBERT, being a highly distilled and compact model (14.3M parameters), was only compressed up to a compression ratio of 0.3 (70\% parameter reduction), as pushing for a higher compression ratio would likely lead to catastrophic performance collapse due to its limited architectural redundancy.}

\input{sections/tables/tinybert}

\added{While the primary focus of this work is post-training compression, it is informative to examine how well the performance of highly compressed models is recovered by subsequent fine-tuning. Table}
\ref{tab:comp-finetune}
\added{presents the results after a single fine-tuning epoch on both MobileBERT and TinyBERT, compressed at 0.5 and 0.3 compression ratio, respectively. Even after introducing the fine-tuning step, NIDA-SVD generally maintains better results than SVD-LLMv2. For MobileBERT, NIDA-SVD registers meaningful gains in \textit{qnli} and \textit{sst2} (both achieving an approximate +1.26\% improvement over SVD-LLMv2). Similarly, on TinyBERT, NIDA-SVD captures an advantage in the inference tasks (\textit{mrpc} and \textit{qnli}), with the \textit{qnli} metric recovering +0.92\% better than the baseline. Overall, across the six task-model pairs evaluated, the NIDA-SVD compressed models yield better accuracy in four instances, underscoring the robustness of our data-aware, importance-driven initialization.}

\input{sections/tables/compression-finetune}

\subsection{Ablation Studies}
\input{sections/tables/fw-ablation}
\input{sections/tables/allocation-ablation-svdllmv2}
\input{sections/tables/allocation-ablation-nida-svd}

In Table \ref{tab:fw-ablation} we provide experimental evidence that under uniform rank allocation, naively combining FWSVD with SVD-LLMv2 (NIDA-SVD (PI)) in 26 out of the 40 cases (more than half of them), the compressed model's task performance is inferior compared to when not using PI at all (in which case the algorithm is equivalent to vanilla SVD-LLMv2). We note that NIDA-SVD (PI) is equivalent to Algorithm \ref{algo:weight_truncation} but with $\hat{\mI}$ computed as discussed in Section \ref{sec:fwsvd}, whereas in NIDA-SVD (NI), $\hat{\mI}$ is computed as discussed in Section \ref{sec:neuron-importance}. Overall, in 38 out of 40 cases (95\%), NI is more effective than PI, and in 35 it improves upon SVD-LLMv2, while remaining competent in the rest of them.

Tables \ref{tab:ablation-rank-alloc-svdllmv2} and \ref{tab:ablation-rank-alloc-nida-svd} summarize experimental results regarding the efficacy of our proposed index-based rank allocation, compared to other strategies that were previously considered.  The experiments show that regardless of the compression algorithm being SVD-LLMv2 or NIDA-SVD, index-based allocation is a better choice for the majority of the cases (more than 67.5\% of them).
Furthermore, a row-wise comparison of Tables
\ref{tab:ablation-rank-alloc-svdllmv2} and \ref{tab:ablation-rank-alloc-nida-svd} reveals that, in most cases, NIDA-SVD outperforms SVD-LLMv2 under any rank allocation algorithm. While for uniform rank allocation this was discussed in the previous section, for role-based allocation, NIDA-SVD outperforms SVD-LLMv2 in 34 cases (85\%) while for index-based allocation in 33 (82.5\%). In many cases the improvement is substantial, especially at high compression ratios: by up to +4.54\% \textbf{absolute improvement} in role-based rank allocation and +7.97\% for index-based rank allocation.

\subsection{Rank Allocation Analysis}
The rank allocation analysis indicates that the index-based strategy (NIDA-SVD) offers advantages over the role-based approach (SVD-LLMv2), mainly because the role-based method handles matrix ranks in a more uniform manner across all transformer layers (especially in Figure \ref{fig:weights-comparison} (a),(b) and(e), resulting in a flatter rank profile that fails to capture the intrinsic heterogeneity in network's sensitivity among layers. This results in role-based ranks being consistently closer to uniform values, reflecting a more uniform compression approach across the network. In contrast, the index-based strategy, as detailed in Section \ref{sec:rank-allocation} and Algorithm \ref{algo:ratio_allocation}, groups all compressible weight types by individual layer index, allowing for a more targeted distribution of ranks that adaptively varies according to the specific redundancy and importance of each layer. This approach enables the index-based method to exploit layer-specific redundancies more effectively, as demonstrated in the plots of Figure \ref{fig:weights-comparison} (a)-(f). Particularly in Figure \ref{fig:weights-comparison} (e), the role-based ranks align much more closely with the uniform values when compared to the index-based ranks.

In Figure \ref{fig:weights-comparison} (a) and (b), the two images illustrate the rank allocations for the down projection and value weight matrices across the 12 layers of the BERT model under varying compression ratios (0.1 to 0.5) using the index-based strategy. For all compression ratios, the behavior of ranks within the same weight type follows a consistent pattern. Specifically, for the $\mW_d$ (Figure \ref{fig:ratio-comparison} (a)), the earlier layers (e.g. 0-5) are more sensitive to compression, requiring more parameters to maintain performance, in contrast to the later layers (e.g. 6-11). Also, for the $\mW_V$ (Figure \ref{fig:ratio-comparison} (b)), the earlier layers are less sensitive to compression than the later layers. As we can see, there are different patterns for different weight types across the layers.

Generally, we cannot assume that a weight will exhibit a specific behavior in rank compression across layers, as the sensitivity to compression can vary depending on the weight type and its index within the model. Our rank allocation algorithm identifies these distinct patterns by analyzing the intrinsic redundancy and importance of each layer, dynamically adjusting the ranks to optimize performance. 

\subsection{\added{Memory and Computational Analysis}}

\input{sections/tables/flops}

\added{The primary motivation for model compression is to minimize memory  and inference time, especially for deployment on resource-constrained devices. To quantify the advances achieved by our low-rank approximation methods, we conducted a complexity analysis measuring both the memory requirements, represented by total parameters (in millions), and the computational cost, measured in MFLOPS/token. This analysis conducted in DistilBERT and MobileBERT for comrpession ratio 0.5, and TinyBERT for compression ratio 0.3, comparing the uncompressed base models against the two SVD-based (SVD-LLMv2 and NIDA-SVD) compression approaches.}

\added{The results presented in Table \protect\ref{tab:flop-analysis}} \added{ clearly demonstrate the reduction in model size. For DistilBERT, compression reduces the total parameters from 66.9 million to approximately 33.4 million, achieving the targeted 50\ reduction. Similarly the memory size of the MobileBERT reduced from 24.5 million to about 12.1 million. This parameter reduction directly translated to a massive decrease in computational cost, as the MFLOPS/token for DistilBERT drops from 86 MFLOPS to just 19 MFLOPS, corresponding to a speedup factor of approximately 4.5. The most dramatic efficiency improvement is seen in TinyBERT, where a 30\% parameter reduction (from 14.3M to 10.0M) results in a spectacular 90\% reduction in computational cost, dropping from $\approx$ 9 MFLOPS/token down to less than 1 MFLOP/token.}

\added{Finaly, the comparison between our NIDA-SVD method and the SVD-LLMv2 method reveals that our importance-aware rank allocation strategy does not impose a measurable computational overhead. Both methods achieve nearly identical parameter counts and, consequently, highly similar MFLOPS/token values across all architectures. This is an expected outcome, as the fundamental change to the network's architecture (replacing a dense matrix $\mW$ with two smaller matrices $\mA$ and $\mB$) is the same for both approaches, regardless of how the compression rate is distributed, since the total computational budget for a fixed compression ratio remains constant. This confirms that in the rank allocation proccess, the performance efficacy of NIDA-SVD in contrast to other SVD-based methods is achieved solely through a smarter distribution of ranks (resources) across layers, not by increasing the computational expense.}

%% file: sections/tables/soa-uniform.tex
\begin{table*}
    \centering
    \captionsetup{font=small}
    \caption{Compressed model performance metrics for different compression ratios (CR) \added{on the BERT base model} under uniform rank allocation. In 35 out of 40 cases (87.5\%), our method (NIDA-SVD) achieves state-of-the-art performance. The best results are in \textbf{bold}, while the second-best are \underline{underlined}.}
    \label{tab:soa-uniform}
    \begin{tabular}{llccccccccc}
\hline
CR & Method & Rank Allocation & cola & mnli-m & mnli-mm & mrpc & qnli & qqp & sst2 & stsb \\ \hline
N/A & BERT base & N/A & 0.562 & 0.847 & 0.849 & 0.906 & 0.916 & 0.88 & 0.923 & 0.891 \\
\hline
\multirow{4}{*}{0.1} & SVD & \multirow{4}{*}{uniform} & 0.2372 & 0.7745 & 0.7363 & 0.3693 & 0.7151 & 0.8043 & 0.8864 & 0.8028 \\
 & FWSVD &  & 0.3488 & 0.8289 & 0.8148 & 0.8877 & 0.8936 & 0.8667 & 0.8933 & 0.8677 \\
 & SVD-LLMv2 &  & \underline{0.5548} & \underline{0.8407} & \textbf{0.8423} & \underline{0.9075} & \underline{0.9115} & \underline{0.8740} & \textbf{0.9174} & \underline{0.8838} \\
 & NIDA-SVD &  & \textbf{0.5600} & \textbf{0.8434} & \underline{0.8418} & \textbf{0.9157} & \textbf{0.9119} & \textbf{0.8747} & \underline{0.9162} & \textbf{0.8846} \\
\hline
\multirow{4}{*}{0.2} & SVD & \multirow{4}{*}{uniform} & 0.1010 & 0.6744 & 0.6183 & 0.0282 & 0.5663 & 0.7145 & 0.8761 & 0.7025 \\
 & FWSVD & & 0.0657 & 0.7987 & 0.7526 & 0.8675 & 0.8500 & 0.8359 & 0.8761 & 0.8375 \\
 & SVD-LLMv2 &  & \textbf{0.5254} & \underline{0.8349} & \underline{0.8335} & \underline{0.9071} & \underline{0.9009} & \underline{0.8652} & \underline{0.9151} & \underline{0.8729} \\
 & NIDA-SVD &  & \underline{0.5094} & \textbf{0.8383} & \textbf{0.8352} & \textbf{0.9128} & \textbf{0.9068} & \textbf{0.8684} & \textbf{0.9185} & \textbf{0.8740} \\
\hline
\multirow{4}{*}{0.3} & SVD & \multirow{4}{*}{uniform} & 0.0358 & 0.4895 & 0.4284 & 0.0000 & 0.4980 & 0.6083 & 0.8245 & 0.5043 \\
 & FWSVD &  & 0.0270 & 0.7065 & 0.6275 & 0.8320 & 0.6760 & 0.7488 & 0.8360 & 0.7871 \\
 & SVD-LLMv2 &  & \textbf{0.4543} & \underline{0.8210} & \underline{0.8152} & \underline{0.8863} & \underline{0.8755} & \underline{0.8514} & \underline{0.9025} & \underline{0.8239} \\
 & NIDA-SVD &  & \underline{0.4451} & \textbf{0.8245} & \textbf{0.8203} & \textbf{0.9003} & \textbf{0.8898} & \textbf{0.8561} & \textbf{0.9048} & \textbf{0.8420} \\
\hline
\multirow{4}{*}{0.4} & SVD & \multirow{4}{*}{uniform} & -0.0028 & 0.3850 & 0.3614 & 0.0000 & 0.4946 & 0.5764 & 0.6559 & 0.4276 \\
 & FWSVD &  & 0.0207 & 0.5562 & 0.5152 & 0.6867 & 0.4978 & 0.6952 & 0.7213 & 0.7102 \\
 & SVD-LLMv2 &  & \underline{0.3206} & \underline{0.7766} & \underline{0.7691} & \underline{0.8376} & \underline{0.8088} & \underline{0.8197} & \underline{0.8876} & \underline{0.7295} \\
 & NIDA-SVD &  & \textbf{0.3271} & \textbf{0.7901} & \textbf{0.7851} & \textbf{0.8657} & \textbf{0.8359} & \textbf{0.8322} & \textbf{0.8979} & \textbf{0.7652} \\
\hline
\multirow{4}{*}{0.5} & SVD & \multirow{4}{*}{uniform} & -0.0252 & 0.3713 & 0.3703 & 0.0000 & 0.4946 & 0.5743 & 0.5126 & 0.2265 \\
 & FWSVD &  & -0.0172 & 0.3541 & 0.4778 & 0.0000 & 0.4949 & 0.4345 & 0.6032 & 0.5981 \\
 & SVD-LLMv2 &  & \underline{0.2079} & \underline{0.6727} & \underline{0.6678} & \underline{0.8176} & \underline{0.6875} & \underline{0.7602} & \textbf{0.8772} & \underline{0.6426} \\
 & NIDA-SVD &  & \textbf{0.2394} & \textbf{0.7143} & \textbf{0.7160} & \textbf{0.8330} & \textbf{0.7113} & \textbf{0.7837} & \underline{0.8704} & \textbf{0.6659} \\
\hline
\end{tabular}
\end{table*}

%% file: sections/tables/vs-svdllm-v2.tex
\begin{table*}
    \centering
    \captionsetup{font=small}
    \caption{Comparison of our method (NIDA-SVD) against the previous state-of-the-art (SVD-LLMv2) \added{for BERT base model} under different compression ratios (CR). Each method is complemented with its own rank allocation algorithm. In 34 out of 40 cases (85\%), our method outperforms SVD-LLMv2, especially in high compression ratios. The best results are in \textbf{bold}.}
    \label{tab:vs-svd-llm-v2}
\begin{tabular}{llccccccccc}
\hline
CR                   & Method    & Rank Allocation & cola            & mnli-m          & mnli-mm         & mrpc            & qnli            & qqp             & sst2            & stsb            \\ \hline
N/A                  & BERT base & N/A             & 0.562           & 0.847           & 0.849           & 0.906           & 0.916           & 0.88            & 0.923           & 0.891           \\ \hline
\multirow{2}{*}{0.1}                 & SVD-LLMv2 & role-based      & \textbf{0.5573} & 0.8419          & 0.8429          & \textbf{0.9087} & \textbf{0.9126} & 0.8736          & 0.9139          & 0.8841          \\
                     & NIDA-SVD  & index-based     & 0.5522          & \textbf{0.8456} & \textbf{0.8444} & 0.9078          & 0.9121          & \textbf{0.8739} & \textbf{0.9162} & \textbf{0.8866} \\ \hline
\multirow{2}{*}{0.2} & SVD-LLMv2 & role-based      & \textbf{0.5325} & 0.8369          & 0.8340          & \textbf{0.9134} & 0.9002          & 0.8657          & 0.9110          & 0.8722          \\
                     & NIDA-SVD  & index-based     & 0.5147          & \textbf{0.8396} & \textbf{0.8387} & 0.9090          & \textbf{0.9068} & \textbf{0.8690} & \textbf{0.9116} & \textbf{0.8793} \\ \hline
\multirow{2}{*}{0.3} & SVD-LLMv2 & role-based      & 0.4556          & 0.8206          & 0.8171          & 0.8907          & 0.8766          & 0.8531          & 0.9002          & 0.8242          \\
                     & NIDA-SVD  & index-based     & \textbf{0.4736} & \textbf{0.8283} & \textbf{0.8224} & \textbf{0.9026} & \textbf{0.8923} & \textbf{0.8576} & \textbf{0.9013} & \textbf{0.8608} \\ \hline
\multirow{2}{*}{0.4} & SVD-LLMv2 & role-based      & 0.3276          & 0.7802          & 0.7715          & 0.8366          & 0.8103          & 0.8199          & 0.8944          & 0.7231          \\
                     & NIDA-SVD  & index-based     & \textbf{0.3486} & \textbf{0.7954} & \textbf{0.7907} & \textbf{0.8708} & \textbf{0.8458} & \textbf{0.8345} & \textbf{0.8956} & \textbf{0.8031} \\ \hline
\multirow{2}{*}{0.5} & SVD-LLMv2 & role-based      & 0.2039          & 0.6762          & 0.6750          & 0.8176          & 0.6917          & 0.7698          & \textbf{0.8727} & 0.6435          \\
                     & NIDA-SVD  & index-based     & \textbf{0.2480} & \textbf{0.7255} & \textbf{0.7251} & \textbf{0.8358} & \textbf{0.7331} & \textbf{0.7890} & 0.8704          & \textbf{0.7076} \\ \hline
\end{tabular}
\end{table*}

%% file: sections/tables/distilbert.tex
\begin{table*}
\centering
\captionsetup{font=small}
\caption{\added{Comparison of our method (NIDA-SVD) against the previous state-of-the-art (SVD-LLMv2) for DistilBERT base model under different compression ratios (CR). Each method is complemented with its own rank allocation algorithm. In 23 out of 25 cases (92\%), our method outperforms SVD-LLMv2, especially in high compression ratios. The best results are in \textbf{bold}.}}
\label{tab:distilbert-comp-compact}
\begin{tabular}{llcccccc}
\hline 
CR & Method & Rank Allocation & mnli-mm & mrpc & qnli & qqp & sst2 \\ 
\hline 
N/A & DistilBERT base & N/A & 0.8232 & 0.8892 & 0.8874 & 0.8643 & 0.9128 \\ 
\hline 
\multirow{2}{*}{0.1} & SVD-LLMv2 & role-based & 0.8147 & 0.8851 & 0.8720 & 0.8569 & 0.9101 \\
& NIDA-SVD & index-based & \textbf{0.8172} & \textbf{0.8877} & \textbf{0.8724} & \textbf{0.8579} & \textbf{0.9105} \\ 
\hline 
\multirow{2}{*}{0.2} & SVD-LLMv2 & role-based & 0.8055 & 0.8866 & 0.8420 & 0.8490 & \textbf{0.9092} \\
& NIDA-SVD & index-based & \textbf{0.8088} & \textbf{0.8870} & \textbf{0.8561} & \textbf{0.8497} & 0.9082 \\ 
\hline 
\multirow{2}{*}{0.3} & SVD-LLMv2 & role-based & 0.7742 & 0.8756 & 0.7761 & 0.8317 & 0.9036 \\
& NIDA-SVD & index-based & \textbf{0.7832} & \textbf{0.8766} & \textbf{0.8004} & \textbf{0.8320} & \textbf{0.9059} \\ 
\hline 
\multirow{2}{*}{0.4} & SVD-LLMv2 & role-based & 0.7016 & 0.8396 & 0.6421 & 0.7591 & 0.8841 \\
& NIDA-SVD & index-based & \textbf{0.7132} & \textbf{0.8467} & \textbf{0.6615} & \textbf{0.7694} & \textbf{0.8876} \\ 
\hline 
\multirow{2}{*}{0.5} & SVD-LLMv2 & role-based & 0.5592 & 0.8040 & 0.5335 & \textbf{0.5101} & 0.8463 \\
& NIDA-SVD & index-based & \textbf{0.5727} & \textbf{0.8053} & \textbf{0.5387} & 0.4919 & \textbf{0.8532} \\ 
\hline 
\end{tabular}
\end{table*}

%% file: sections/tables/mobilebert.tex
\begin{table*}
\centering
\captionsetup{font=small}
\caption{\added{Comparison of our method (NIDA-SVD) against the previous state-of-the-art (SVD-LLMv2) for MobileBERT base model under different compression ratios (CR). Each method is complemented with its own rank allocation algorithm. In 22 out of 25 cases (88\%), our method outperforms SVD-LLMv2, especially in high compression ratios. The best results are in \textbf{bold}.}}
\label{tab:mobilebert-comp-compact}
\begin{tabular}{llcccccc} 
\hline
CR & Method & Rank Allocation & mrpc & qnli & qqp & sst2 & stsb \\ 
\hline
N/A & MobileBERT base & N/A & 0.8888 & 0.9068 & 0.8670 & 0.9128 & 0.8773 \\ 
\hline
\multirow{2}{*}{0.1} & SVD-LLMv2 & role-based & \textbf{0.8625} & 0.8224 & 0.7856 & 0.8830 & 0.8191 \\
& NIDA-SVD & index-based & 0.8603 & \textbf{0.8506} & \textbf{0.8042} & \textbf{0.8841} & \textbf{0.8452} \\ 
\hline
\multirow{2}{*}{0.2} & SVD-LLMv2 & role-based & 0.8277 & 0.7940 & 0.7437 & 0.8337 & 0.7741 \\
& NIDA-SVD & index-based & \textbf{0.8290} & \textbf{0.8065} & \textbf{0.7664} & \textbf{0.8451} & \textbf{0.7972} \\ 
\hline
\multirow{2}{*}{0.3} & SVD-LLMv2 & role-based & 0.8218 & 0.7298 & 0.6971 & 0.8027 & 0.7570 \\
& NIDA-SVD & index-based & \textbf{0.8284} & \textbf{0.7346} & \textbf{0.7339} & \textbf{0.8130} & \textbf{0.7773} \\ 
\hline
\multirow{2}{*}{0.4} & SVD-LLMv2 & role-based & 0.8183 & 0.5976 & 0.6750 & 0.7717 & 0.6691 \\
& NIDA-SVD & index-based & \textbf{0.8210} & \textbf{0.5991} & \textbf{0.6782} & \textbf{0.7740} & \textbf{0.7053} \\ 
\hline
\multirow{2}{*}{0.5} & SVD-LLMv2 & role-based & 0.6895 & 0.5081 & \textbf{0.5834} & 0.7419 & \textbf{0.4664} \\
& NIDA-SVD & index-based & \textbf{0.6919} & \textbf{0.5092} & 0.5735 & \textbf{0.7488} & 0.3867 \\ 
\hline
\end{tabular}
\end{table*}

%% file: sections/tables/tinybert.tex
\begin{table*}
\centering
\captionsetup{font=small}
\caption{\added{Comparison of our method (NIDA-SVD) against the previous state-of-the-art (SVD-LLMv2) for TinyBERT model under different compression ratios (CR). Each method is complemented with its own rank allocation algorithm. In 17 out of 18 cases (94\%), our method outperforms SVD-LLMv2, especially in high compression ratios. The best results are in \textbf{bold}.}}
\label{tab:tinybert-comp-compact}
\begin{tabular}{llccccccc}
\hline
CR                   & Method    & Rank Allocation & mnli-mm         & mrpc            & qnli            & qqp             & sst2            & stsb            \\ \hline
N/A                  & TinyBERT base & N/A             & 0.7955          & 0.8881          & 0.837           & 0.8512          & 0.8704          & 0.8731          \\ \hline
\multirow{2}{*}{0.1} & SVD-LLMv2 & role-based      & 0.7759          & 0.8722          & 0.8107          & 0.8405          & 0.8612          & 0.8647          \\
                     & NIDA-SVD  & index-based     & \textbf{0.778}  & \textbf{0.8736} & \textbf{0.8125} & \textbf{0.8421} & \textbf{0.8646} & \textbf{0.8672} \\ \hline
\multirow{2}{*}{0.2} & SVD-LLMv2 & role-based      & 0.6456          & 0.8278          & 0.6761          & 0.7837          & \textbf{0.8486} & 0.8202          \\
                     & NIDA-SVD  & index-based     & \textbf{0.6554} & \textbf{0.8291} & \textbf{0.6904} & \textbf{0.7891} & 0.8451          & \textbf{0.8233} \\ \hline
\multirow{2}{*}{0.3} & SVD-LLMv2 & role-based      & 0.3721          & 0.8122          & 0.5215          & 0.5406          & 0.5676          & 0.1146          \\
                     & NIDA-SVD  & index-based     & \textbf{0.3802} & \textbf{0.8122} & \textbf{0.5216} & \textbf{0.5422} & \textbf{0.5745} & \textbf{0.115}  \\ \hline
\end{tabular}
\end{table*}

%% file: sections/tables/compression-finetune.tex
\begin{table}
\centering
\captionsetup{font=small}
\caption{\added{Single-epoch fine-tuning comparison of NIDA-SVD and SVD-LLMv2 on MobileBERT at 0.5 compression ration and TinyBERT at 0.3 compression ratio. The results (F1/Accuracy on MRPC, QNLI, SST-2) show that fine-tuning after NIDA-SVD compression achieves better accuracy in 4 out of 6 cases (66.6\%), demonstrating robust performance despite the substantial compression level.}}
\label{tab:comp-finetune}
\begin{tabular}{lllll}
\hline
CR                   & Method     & mrpc            & qnli            & sst2            \\ \hline
N/A                  & MobileBERT & 0.8888          & 0.9068          & 0.9036          \\ \hline
\multirow{2}{*}{0.5} & SVD-LLMv2  & \textbf{0.8145} & 0.8281          & 0.8543          \\
                     & NIDA-SVD   & 0.8096          & \textbf{0.8407} & \textbf{0.8669} \\ \hline
N/A                  & TinyBERT   & 0.8881          & 0.837           & 0.8704          \\ \hline
\multirow{2}{*}{0.3} & SVD-LLMv2  & 0.816           & 0.6148          & \textbf{0.8325} \\
                     & NIDA-SVD   & \textbf{0.8222} & \textbf{0.624}  & 0.8176          \\ \hline
\end{tabular}
\end{table}

%% file: sections/tables/fw-ablation.tex
\begin{table*}
    \centering
    \captionsetup{font=small}
    \caption{Ablation Study under different compression ratios (CR): NIDA-SVD (ours) using Parameter Importance (PI) as suggested in FWSVD, and our proposed Neuron Importance (NI). In 38 out of 40 cases (95\%), NI is more effective than PI and in 35 (87.5\%) it improves upon SVD-LLMv2. In 26 out of 40 cases (65\%), using PI makes the results worse than not using it at all (NIDA-SVD (PI) vs SVD-LLMv2). The best results are in \textbf{bold}, while the second-best are \underline{underlined}.}
    \label{tab:fw-ablation}
    \begin{tabular}{llccccccccc}
\hline
CR & Method & Rank Allocation & cola & mnli-m & mnli-mm & mrpc & qnli & qqp & sst2 & stsb \\ \hline
\multirow{3}{*}{0.1} & SVD-LLMv2 & \multirow{3}{*}{uniform} & \underline{0.5548} & 0.8407 & \textbf{0.8423} & 0.9075 & \underline{0.9115} & \underline{0.8740} & \textbf{0.9174} & \underline{0.8838} \\
 & NIDA-SVD (PI) &  & 0.5482 & \underline{0.8416} & 0.8389 & \underline{0.9109} & 0.9108 & 0.8734 & 0.9139 & 0.8830 \\
 & NIDA-SVD (NI) &  & \textbf{0.5600} & \textbf{0.8434} & \underline{0.8418} & \textbf{0.9157} & \textbf{0.9119} & \textbf{0.8747} & \underline{0.9160} & \textbf{0.8846} \\
\hline
\multirow{3}{*}{0.2} & SVD-LLMv2 & \multirow{3}{*}{uniform} & \textbf{0.5254} & \underline{0.8349} & \underline{0.8335} & 0.9071 & 0.9009 & 0.8652 & \underline{0.9151} & \underline{0.8729} \\
 & NIDA-SVD (PI) &  & 0.4975 & 0.8334 & 0.8284 & \textbf{0.9159} & \underline{0.9033} & \underline{0.8658} & 0.9048 & 0.8704 \\
 & NIDA-SVD (NI) &  & \underline{0.5094} & \textbf{0.8383} & \textbf{0.8352} & \underline{0.9128} & \textbf{0.9068} & \textbf{0.8684} & \textbf{0.9185} & \textbf{0.8740} \\
\hline
\multirow{3}{*}{0.3} & SVD-LLMv2 & \multirow{3}{*}{uniform} & \textbf{0.4543} & \underline{0.8210} & \underline{0.8152} & 0.8863 & 0.8755 & \underline{0.8514} & \underline{0.9025} & 0.8239 \\
 & NIDA-SVD (PI) &  & 0.3650 & 0.8141 & 0.8022 & \underline{0.8970} & \underline{0.8791} & 0.8489 & 0.8922 & \underline{0.8285} \\
 & NIDA-SVD (NI) &  & \underline{0.4451} & \textbf{0.8245} & \textbf{0.8203} & \textbf{0.9003} & \textbf{0.8898} & \textbf{0.8561} & \textbf{0.9048} & \textbf{0.8420} \\
\hline
\multirow{3}{*}{0.4} & SVD-LLMv2 & \multirow{3}{*}{uniform} & \underline{0.3206} & \underline{0.7766} & \underline{0.7691} & 0.8376 & 0.8088 & \underline{0.8197} & 0.8876 & 0.7295 \\
 & NIDA-SVD (PI) &  & 0.2937 & 0.7687 & 0.7473 & \underline{0.8652} & \underline{0.8160} & 0.8174 & \underline{0.8887} & \underline{0.7433} \\
 & NIDA-SVD (NI) &  & \textbf{0.3271} & \textbf{0.7901} & \textbf{0.7851} & \textbf{0.8657} & \textbf{0.8359} & \textbf{0.8322} & \textbf{0.8979} & \textbf{0.7652} \\
\hline
\multirow{3}{*}{0.5} & SVD-LLMv2 & \multirow{3}{*}{uniform} & \underline{0.2079} & \underline{0.6727} & \underline{0.6678} & 0.8176 & \underline{0.6875} & 0.7602 & \textbf{0.8772} & \underline{0.6426} \\
 & NIDA-SVD (PI) &  & 0.1932 & 0.6685 & 0.6615 & \textbf{0.8384} & 0.6732 & \underline{0.7612} & 0.8566 & 0.6024 \\
 & NIDA-SVD (NI) &  & \textbf{0.2394} & \textbf{0.7143} & \textbf{0.7160} & \underline{0.8330} & \textbf{0.7113} & \textbf{0.7837} & \underline{0.8704} & \textbf{0.6659} \\
\hline
\end{tabular}
\end{table*}

%% file: sections/tables/allocation-ablation-svdllmv2.tex
\begin{table*}
    \centering
    \captionsetup{font=small}
    \caption{Ablation Study: Our index-based rank allocation algorithm improves vanilla SVD-LLMv2 compared to other allocation strategies in 27 out of 40 cases (67.5\%). The best results are in \textbf{bold}, while the second-best are \underline{underlined}.}
    \label{tab:ablation-rank-alloc-svdllmv2}
\begin{tabular}{llccccccccc}
\hline
CR & Method & Rank Allocation & cola & mnli-m & mnli-mm & mrpc & qnli & qqp & sst2 & stsb \\ \hline
\multirow{3}{*}{0.1} & \multirow{3}{*}{SVD-LLMv2} & uniform & \underline{0.5548} & 0.8407 & 0.8423 & 0.9075 & \underline{0.9115} & \textbf{0.8740} & \textbf{0.9174} & 0.8838 \\
 &  & role-based & \textbf{0.5573} & \underline{0.8419} & \underline{0.8429} & \underline{0.9087} & \textbf{0.9126} & 0.8736 & 0.9139 & \underline{0.8841} \\
 &  & index-based & 0.5490 & \textbf{0.8420} & \textbf{0.8435} & \textbf{0.9089} & 0.9088 & \underline{0.8737} & \underline{0.9174} & \textbf{0.8842} \\
\hline
\multirow{3}{*}{0.2} & \multirow{3}{*}{SVD-LLMv2} & uniform & 0.5254 & 0.8349 & 0.8335 & 0.9071 & \underline{0.9009} & 0.8652 & \textbf{0.9151} & \textbf{0.8729} \\
 &  & role-based & \underline{0.5325} & \underline{0.8369} & \underline{0.8340} & \textbf{0.9134} & 0.9002 & \underline{0.8657} & 0.9110 & 0.8722 \\
 &  & index-based & \textbf{0.5334} & \textbf{0.8387} & \textbf{0.8356} & \underline{0.9115} & \textbf{0.9028} & \textbf{0.8670} & \underline{0.9116} & \underline{0.8727} \\
\hline
\multirow{3}{*}{0.3} & \multirow{3}{*}{SVD-LLMv2} & uniform & 0.4543 & \underline{0.8210} & 0.8152 & 0.8863 & 0.8755 & 0.8514 & \textbf{0.9025} & \underline{0.8239} \\
 &  & role-based & \underline{0.4556} & 0.8206 & \underline{0.8171} & \textbf{0.8907} & \underline{0.8766} & \underline{0.8531} & 0.9002 & \textbf{0.8242} \\
 &  & index-based & \textbf{0.4796} & \textbf{0.8255} & \textbf{0.8179} & \underline{0.8874} & \textbf{0.8810} & \textbf{0.8544} & \underline{0.9005} & 0.8237 \\
\hline
\multirow{3}{*}{0.4} & \multirow{3}{*}{SVD-LLMv2} & uniform & 0.3206 & 0.7766 & 0.7691 & \underline{0.8376} & 0.8088 & 0.8197 & 0.8876 & \textbf{0.7295} \\
 &  & role-based & \underline{0.3276} & \underline{0.7802} & \underline{0.7715} & 0.8366 & \underline{0.8103} & \underline{0.8199} & \underline{0.8944} & 0.7231 \\
 &  & index-based & \textbf{0.3446} & \textbf{0.7835} & \textbf{0.7770} & \textbf{0.8496} & \textbf{0.8222} & \textbf{0.8265} & \textbf{0.8948} & \underline{0.7234} \\
\hline
\multirow{3}{*}{0.5} & \multirow{3}{*}{SVD-LLMv2} & uniform & \underline{0.2079} & 0.6727 & 0.6678 & \underline{0.8176} & 0.6875 & 0.7602 & \textbf{0.8772} & 0.6426 \\
 &  & role-based & 0.2039 & \underline{0.6762} & \underline{0.6750} & 0.8176 & \underline{0.6917} & \textbf{0.7698} & 0.8727 & \underline{0.6435} \\
 &  & index-based & \textbf{0.2336} & \textbf{0.6846} & \textbf{0.6797} & \textbf{0.8237} & \textbf{0.7054} & \underline{0.7657} & \underline{0.8772} & \textbf{0.6441} \\
\hline
\end{tabular}
\end{table*}

%% file: sections/tables/allocation-ablation-nida-svd.tex
\begin{table*}
    \centering
    \captionsetup{font=small}
    \caption{Ablation Study: Our index-based rank allocation algorithm improves NIDA-SVD compared to other allocation strategies in 28 out of 40 cases (70\%). The best results are in \textbf{bold}, while the second-best are \underline{underlined}.}
    \label{tab:ablation-rank-alloc-nida-svd}
\begin{tabular}{llccccccccc}
\hline
CR & Method & Rank Allocation & cola & mnli-m & mnli-mm & mrpc & qnli & qqp & sst2 & stsb \\ \hline
\multirow{3}{*}{0.1} & \multirow{3}{*}{NIDA-SVD} & uniform & \underline{0.5600} & 0.8434 & \underline{0.8418} & \textbf{0.9157} & \underline{0.9119} & \underline{0.8747} & 0.9160 & 0.8846 \\
 &  & role-based & \textbf{0.5656} & \underline{0.8450} & 0.8416 & \underline{0.9141} & 0.9110 & \textbf{0.8749} & \textbf{0.9174} & \underline{0.8852} \\
 &  & index-based & 0.5522 & \textbf{0.8456} & \textbf{0.8444} & 0.9078 & \textbf{0.9121} & 0.8739 & \underline{0.9162} & \textbf{0.8866} \\
\hline
\multirow{3}{*}{0.2} & \multirow{3}{*}{NIDA-SVD} & uniform & 0.5094 & 0.8383 & 0.8352 & \textbf{0.9128} & \textbf{0.9068} & 0.8684 & \textbf{0.9185} & 0.8740 \\
 &  & role-based & \textbf{0.5198} & \underline{0.8386} & \underline{0.8363} & 0.9056 & 0.9066 & \underline{0.8687} & \underline{0.9185} & \underline{0.8754} \\
 &  & index-based & \underline{0.5147} & \textbf{0.8396} & \textbf{0.8387} & \underline{0.9090} & \underline{0.9068} & \textbf{0.8690} & 0.9116 & \textbf{0.8793} \\
\hline
\multirow{3}{*}{0.3} & \multirow{3}{*}{NIDA-SVD} & uniform & \underline{0.4451} & 0.8245 & 0.8203 & 0.9003 & \underline{0.8898} & 0.8561 & \underline{0.9048} & 0.8420 \\
 & & role-based & 0.4365 & \underline{0.8247} & \underline{0.8221} & \textbf{0.9063} & 0.8896 & \underline{0.8569} & \textbf{0.9059} & \underline{0.8433} \\
 &  & index-based & \textbf{0.4736} & \textbf{0.8283} & \textbf{0.8224} & \underline{0.9026} & \textbf{0.8923} & \textbf{0.8576} & 0.9013 & \textbf{0.8608} \\
\hline
\multirow{3}{*}{0.4} & \multirow{3}{*}{NIDA-SVD} & uniform & 0.3271 & 0.7901 & 0.7851 & 0.8657 & 0.8359 & 0.8322 & \textbf{0.8979} & 0.7652 \\
 &  & role-based & \underline{0.3325} & \underline{0.7911} & \underline{0.7879} & \underline{0.8698} & \underline{0.8367} & \underline{0.8329} & \underline{0.8967} & \underline{0.7685} \\
 &  & index-based & \textbf{0.3486} & \textbf{0.7954} & \textbf{0.7907} & \textbf{0.8708} & \textbf{0.8458} & \textbf{0.8345} & 0.8956 & \textbf{0.8031} \\
\hline
\multirow{3}{*}{0.5} & \multirow{3}{*}{NIDA-SVD} & uniform & 0.2394 & 0.7143 & 0.7160 & 0.8330 & 0.7113 & 0.7837 & \underline{0.8704} & \underline{0.6659} \\
 &  & role-based & \underline{0.2424} & \underline{0.7152} & \underline{0.7163} & \underline{0.8348} & \underline{0.7127} & \underline{0.7849} & \textbf{0.8715} & 0.6648 \\
 &  & index-based & \textbf{0.2480} & \textbf{0.7255} & \textbf{0.7251} & \textbf{0.8358} & \textbf{0.7331} & \textbf{0.7890} & 0.8704 & \textbf{0.7076} \\
\hline
\end{tabular}
\end{table*}

%% file: sections/tables/flops.tex
\begin{table*}
\centering
\captionsetup{font=small}
\caption{\added{Computational and Memory Analysis. Comparison of MLFLOPS/token and total parameters (in Millions) for base and compressed DistilBERT, MobileBERT, and TinyBERT models}}
\label{tab:flop-analysis}
\begin{tabular}{llllcc}
\hline
Architecture                & GLUE                  & CR                   & Compression Method & \multicolumn{1}{l}{MLFLOPS/token} & \multicolumn{1}{l}{Total Parameters} \\ \hline
\multirow{9}{*}{distilBERT} & \multirow{3}{*}{mrpc} & N/A                  & base               & 85.96                                 & 66.9                                 \\ \cline{3-6} 
                            &                       & \multirow{2}{*}{0.5} & SVD-LLMv2          & 18.99                                 & 33.4                                 \\
                            &                       &                      & NIDA-SVD           & 18.99                                 & 33.4                                 \\ \cline{2-6} 
                            & \multirow{3}{*}{qnli} & N/A                  & base               & 85.97                                 & 66.9                                 \\ \cline{3-6} 
                            &                       & \multirow{2}{*}{0.5} & SVD-LLMv2          & 18.99                                 & 33.4                                 \\
                            &                       &                      & NIDA-SVD           & 19.07                                 & 33.5                                 \\ \cline{2-6} 
                            & \multirow{3}{*}{sst2} & N/A                  & base               & 85.56                                 & 66.9                                 \\ \cline{3-6} 
                            &                       & \multirow{2}{*}{0.5} & SVD-LLMv2          & 18.59                                 & 33.4                                 \\
                            &                       &                      & NIDA-SVD           & 18.73                                 & 33.5                                 \\ \hline
\multirow{9}{*}{mobileBERT} & \multirow{3}{*}{mrpc} & N/A                  & base               & 41.17                                 & 24.5                                 \\ \cline{3-6} 
                            &                       & \multirow{2}{*}{0.5} & SVD-LLMv2          & 16.33                                 & 12.1                                 \\
                            &                       &                      & NIDA-SVD           & 16.48                                 & 12.2                                 \\ \cline{2-6} 
                            & \multirow{3}{*}{qnli} & N/A                  & base               & 41.18                                 & 24.5                                 \\ \cline{3-6} 
                            &                       & \multirow{2}{*}{0.5} & SVD-LLMv2          & 16.33                                 & 12.1                                 \\
                            &                       &                      & NIDA-SVD           & 16.48                                 & 12.2                                 \\ \cline{2-6} 
                            & \multirow{3}{*}{sst2} & N/A                  & base               & 40.89                                 & 24.5                                 \\ \cline{3-6} 
                            &                       & \multirow{2}{*}{0.5} & SVD-LLMv2          & 16.05                                 & 12.1                                 \\
                            &                       &                      & NIDA-SVD           & 16.24                                 & 12.2                                 \\ \hline
\multirow{9}{*}{tinyBERT}   & \multirow{3}{*}{mrpc} & N/A                  & base               & 9.38                                  & 14.3                                 \\ \cline{3-6} 
                            &                       & \multirow{2}{*}{0.3} & SVD-LLMv2          & 0.75                                  & 10.0                                 \\
                            &                       &                      & NIDA-SVD           & 0.73                                  & 10.0                                 \\ \cline{2-6} 
                            & \multirow{3}{*}{qnli} & N/A                  & base               & 9.38                                  & 14.3                                 \\ \cline{3-6} 
                            &                       & \multirow{2}{*}{0.3} & SVD-LLMv2          & 0.75                                  & 10.0                                 \\
                            &                       &                      & NIDA-SVD           & 0.73                                  & 10.0                                 \\ \cline{2-6} 
                            & \multirow{3}{*}{sst2} & N/A                  & base               & 9.27                                  & 14.3                                 \\ \cline{3-6} 
                            &                       & \multirow{2}{*}{0.3} & SVD-LLMv2          & 0.64                                  & 10.0                                 \\
                            &                       &                      & NIDA-SVD           & 0.62                                  & 10.0                                 \\ \hline
\end{tabular}
\end{table*}

%% file: sections/conclusion.tex
In summary, this work introduced NIDA-SVD, a novel framework to compress LLM matrices that integrates neuron importance with data-aware low-rank approximation, alongside a computationally efficient algorithm for dynamic rank allocation. Experimental results demonstrate that NIDA-SVD consistently outperforms prior state-of-the-art methods under diverse allocation strategies, while our proposed rank allocation algorithm also independently strengthens both NIDA-SVD and previous approaches. Together, these contributions establish a robust and versatile compression solution, achieving substantial performance gains, especially at high compression ratios, paving the way for a more effective deployment of large-scale language models in resource-constrained setups.

%% file: bios/atdovas/bio.tex
\begin{IEEEbiography}[{\includegraphics[width=1in,height=1.25in,clip,keepaspectratio]{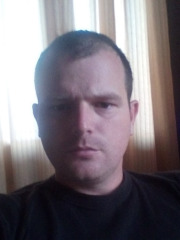}}]{Athanasios Ntovas}{\space} 
received his diploma from the Computers and Communication Engineering Department of the University of Thessaly (UTH) in 2019.

Since April 2020, he has been working as a research assistant in the Information Technologies Institute (ITI) of the Centre for Research and Technology Hellas (CERTH).
His research interests include artificial intelligence, computer vision, and digital signal processing. 
\end{IEEEbiography}

%% file: bios/aldoum/bio.tex
\begin{IEEEbiography}[{\includegraphics[width=1in,height=1.25in,clip,keepaspectratio]{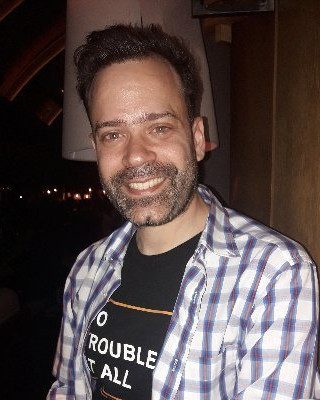}}]{Alexandros Doumanoglou}{\space} has received the diploma of Electrical and Computer Engineer from the Aristotle University of Thessaloniki (A.U.Th) and is currently doing his PhD in Explainable and Interpretable Artificial Intelligence at the department of Advanced Computing Sciences at Maastricht University, under the supervision of Prof. Kurt Driessens.

He joined the Information Technologies Institute in 2012, and since then he has been working as a research assistant in the fields of computer vision, 3D graphics, and machine learning. His current research interests include computer vision, unsupervised learning, representation learning, mechanistic interpretability, and explainable and interpretable methods for deep learning models.
\end{IEEEbiography}

%% file: bios/drak/bio.tex
\begin{IEEEbiography}[{\includegraphics[width=1in,height=1.25in,clip,keepaspectratio]{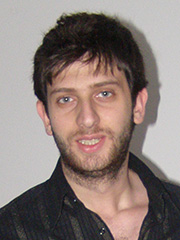}}]{Petros Drakoulis}{\space} received his BSc in Informatics from International Hellenic University and his MSc in Digital Media and Computational Intelligence from Aristotle University of Thessaloniki.

In 2018, he joined the Visual Computing Laboratory of CERTH-ITI where he works as a Research Associate ever since. His main areas of interest include Software Engineering, Visual Computing, Machine Learning and Graphics.
\end{IEEEbiography}

%% file: bios/zarpalas/bio.tex
\begin{IEEEbiography}[{\includegraphics[width=1in,height=1.25in,clip,keepaspectratio]{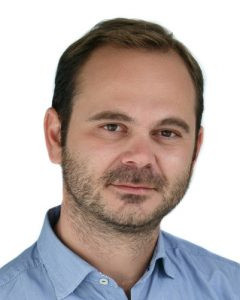}}]{Dimitris Zarpalas}{\space} holds the diploma of Electrical and Computer Engineer from Aristotle University of Thessaloniki, A.U.Th, an MSc in Electrical Engineering (focusing on computer vision) from The Pennsylvania State University, and a PhD in medical informatics (Health Science School, department of Medicine, A.U.Th).

He has joined the Information Technologies Institute in 2007, and is currently a researcher, grade C. His research interests include real time tele-immersion applications (3D reconstruction of moving
humans and their compression), 3D computer vision, 3D medical image processing, shape analysis of anatomical structures, 3D object recognition, motion capturing and evaluation, while in the past has also worked in indexing, search and retrieval, classification of 3D objects and 3D model watermarking.
\end{IEEEbiography}

%% file: refs.bib
@article{drone,
	title        = {Drone: Data-aware low-rank compression for large nlp models},
	author       = {Chen, Patrick and Yu, Hsiang-Fu and Dhillon, Inderjit and Hsieh, Cho-Jui},
	year         = 2021,
	journal      = {Advances in neural information processing systems},
	volume       = 34,
	pages        = {29321--29334}
}

@inproceedings{fwsvd,
	title        = {Language model compression with weighted low-rank factorization},
	author       = {Hsu, Yen-Chang and Hua, Ting and Chang, Sungen and Lou, Qian and Shen, Yilin and Jin, Hongxia},
	booktitle    = {International Conference on Learning Representations},
    year         = 2022,
}

@article{asvd,
	title        = {Asvd: Activation-aware singular value decomposition for compressing large language models},
	author       = {Yuan, Zhihang and Shang, Yuzhang and Song, Yue and Wu, Qiang and Yan, Yan and Sun, Guangyu},
	year         = 2023,
	journal      = {arXiv preprint arXiv:2312.05821}
}

@inproceedings{svd-llm,
	title        = {SVD-LLM: Truncation-aware Singular Value Decomposition for Large Language Model Compression},
	author       = {Wang, Xin and Zheng, Yu and Wan, Zhongwei and Zhang, Mi},
	booktitle    = {The Thirteenth International Conference on Learning Representations},
    year         = 2024,
}

@inproceedings{svd-llmv2,
	title        = {SVD-LLM V2: Optimizing Singular Value Truncation for Large Language Model Compression},
	author       = {Wang, Xin and Alam, Samiul and Wan, Zhongwei and Shen, Hui and Zhang, Mi},
	year         = 2025,
	booktitle    = {Proceedings of the 2025 Conference of the Nations of the Americas Chapter of the Association for Computational Linguistics: Human Language Technologies (Volume 1: Long Papers)},
	pages        = {4287--4296}
}

@article{comp-survey1,
	title        = {Compression of deep learning models for text: A survey},
	author       = {Gupta, Manish and Agrawal, Puneet},
	year         = 2022,
	journal      = {ACM Transactions on Knowledge Discovery from Data (TKDD)},
	publisher    = {ACM New York, NY},
	volume       = 16,
	number       = 4,
	pages        = {1--55}
}

@article{comp-survey2,
	title        = {A survey on model compression for large language models},
	author       = {Zhu, Xunyu and Li, Jian and Liu, Yong and Ma, Can and Wang, Weiping},
	year         = 2024,
	journal      = {Transactions of the Association for Computational Linguistics},
	publisher    = {MIT Press 255 Main Street, 9th Floor, Cambridge, Massachusetts 02142, USA~…},
	volume       = 12,
	pages        = {1556--1577}
}

@article{comp-survey3,
  title={Low-rank matrix completion: A contemporary survey},
  author={Nguyen, Luong Trung and Kim, Junhan and Shim, Byonghyo},
  journal={IEEE Access},
  volume={7},
  pages={94215--94237},
  year={2019},
  publisher={IEEE}
}

@article{comp-survey4,
  title={Learning and compressing: Low-rank matrix factorization for deep neural network compression},
  author={Cai, Gaoyuan and Li, Juhu and Liu, Xuanxin and Chen, Zhibo and Zhang, Haiyan},
  journal={Applied Sciences},
  volume={13},
  number={4},
  pages={2704},
  year={2023},
  publisher={MDPI}
}

@inproceedings{bert,
	title        = {Bert: Pre-training of deep bidirectional transformers for language understanding},
	author       = {Devlin, Jacob and Chang, Ming-Wei and Lee, Kenton and Toutanova, Kristina},
	year         = 2019,
	booktitle    = {Proceedings of the 2019 conference of the North American chapter of the association for computational linguistics: human language technologies, volume 1 (long and short papers)},
	pages        = {4171--4186}
}

@article{gpt,
	title        = {Improving language understanding by generative pre-training},
	author       = {Radford, Alec and Narasimhan, Karthik and Salimans, Tim and Sutskever, Ilya and others},
	year         = 2018,
	publisher    = {San Francisco, CA, USA}
}

@article{gemma2,
  title={Gemma 2: Improving Open Language Models at a Practical Size},
  author={Rivi{\`e}re, Morgane and Pathak, Shreya and Sessa, Pier Giuseppe and Hardin, Cassidy and Bhupatiraju, Surya and Hussenot, L{\'e}onard and Mesnard, Thomas and Shahriari, Bobak and Ram{\'e}, Alexandre and Ferret, Johan and others},
  journal={CoRR},
  year={2024}
}

@article{gemma3,
  publtype={informal},
  author={Aishwarya Kamath and Johan Ferret and Shreya Pathak and Nino Vieillard and Ramona Merhej and Sarah Perrin and Tatiana Matejovicova and Alexandre Ramé and Morgane Rivière and Louis Rouillard and Thomas Mesnard and Geoffrey Cideron and Jean-Bastien Grill and Sabela Ramos and Edouard Yvinec and Michelle Casbon and Etienne Pot and Ivo Penchev and Gaël Liu and Francesco Visin and Kathleen Kenealy and Lucas Beyer and Xiaohai Zhai and Anton Tsitsulin and Róbert Busa-Fekete and Alex Feng and Noveen Sachdeva and Benjamin Coleman and Yi Gao and Basil Mustafa and Iain Barr and Emilio Parisotto and David Tian and Matan Eyal and Colin Cherry and Jan-Thorsten Peter and Danila Sinopalnikov and Surya Bhupatiraju and Rishabh Agarwal and Mehran Kazemi and Dan Malkin and Ravin Kumar and David Vilar and Idan Brusilovsky and Jiaming Luo and Andreas Steiner and Abe Friesen and Abhanshu Sharma and Abheesht Sharma and Adi Mayrav Gilady and Adrian Goedeckemeyer and Alaa Saade and Alexander Kolesnikov and Alexei Bendebury and Alvin Abdagic and Amit Vadi and András György and André Susano Pinto and Anil Das and Ankur Bapna and Antoine Miech and Antoine Yang and Antonia Paterson and Ashish Shenoy and Ayan Chakrabarti and Bilal Piot and Bo Wu and Bobak Shahriari and Bryce Petrini and Charlie Chen and Charline Le Lan and Christopher A. Choquette-Choo and CJ Carey and Cormac Brick and Daniel Deutsch and Danielle Eisenbud and Dee Cattle and Derek Cheng and Dimitris Paparas and Divyashree Shivakumar Sreepathihalli and Doug Reid and Dustin Tran and Dustin Zelle and Eric Noland and Erwin Huizenga and Eugene Kharitonov and Frederick Liu and Gagik Amirkhanyan and Glenn Cameron and Hadi Hashemi and Hanna Klimczak-Plucinska and Harman Singh and Harsh Mehta and Harshal Tushar Lehri and Hussein Hazimeh and Ian Ballantyne and Idan Szpektor and Ivan Nardini},
  title={Gemma 3 Technical Report},
  year={2025},
  month={March},
  cdate={1740787200000},
  journal={CoRR},
  volume={abs/2503.19786},
  url={https://doi.org/10.48550/arXiv.2503.19786}
}

@article{llama,
  author       = {Hugo Touvron and
                  Thibaut Lavril and
                  Gautier Izacard and
                  Xavier Martinet and
                  Marie{-}Anne Lachaux and
                  Timoth{\'{e}}e Lacroix and
                  Baptiste Rozi{\`{e}}re and
                  Naman Goyal and
                  Eric Hambro and
                  Faisal Azhar and
                  Aur{\'{e}}lien Rodriguez and
                  Armand Joulin and
                  Edouard Grave and
                  Guillaume Lample},
  title        = {LLaMA: Open and Efficient Foundation Language Models},
  journal      = {CoRR},
  volume       = {abs/2302.13971},
  year         = {2023},
  url          = {https://doi.org/10.48550/arXiv.2302.13971},
  doi          = {10.48550/ARXIV.2302.13971},
  eprinttype    = {arXiv},
  eprint       = {2302.13971},
  bibsource    = {dblp computer science bibliography, https://dblp.org}
}

@article{claude,
	title        = {Claude 2.0 large language model: Tackling a real-world classification problem with a new iterative prompt engineering approach},
	author       = {Caruccio, Loredana and Cirillo, Stefano and Polese, Giuseppe and Solimando, Giandomenico and Sundaramurthy, Shanmugam and Tortora, Genoveffa},
	year         = 2024,
	journal      = {Intelligent Systems with Applications},
	publisher    = {Elsevier},
	volume       = 21,
	pages        = 200336
}

@article{prunning-1,
	title        = {Sparsity in deep learning: Pruning and growth for efficient inference and training in neural networks},
	author       = {Hoefler, Torsten and Alistarh, Dan and Ben-Nun, Tal and Dryden, Nikoli and Peste, Alexandra},
	year         = 2021,
	journal      = {Journal of Machine Learning Research},
	volume       = 22,
	number       = 241,
	pages        = {1--124}
}

@inproceedings{sparse-gpt-prune,
	title        = {Sparsegpt: Massive language models can be accurately pruned in one-shot},
	author       = {Frantar, Elias and Alistarh, Dan},
	year         = 2023,
	booktitle    = {International conference on machine learning},
	pages        = {10323--10337},
	organization = {PMLR}
}

@inproceedings{prune-slicegpt,
      title = {SliceGPT: Compress Large Language Models by Deleting Rows and Columns},
      author = {Saleh Ashkboos and Maximilian L. Croci and Marcelo Gennari do Nascimento and Torsten Hoefler and James Hensman},
      booktitle = {International Conference on Learning Representations (ICLR)},
      year = {2024}
}

@inproceedings{llm-surgeon-prune,
	title        = {The LLM Surgeon},
	author       = {van der Ouderaa, Tycho FA and Nagel, Markus and Van Baalen, Mart and Blankevoort, Tijmen},
	booktitle    = {The Twelfth International Conference on Learning Representations}
}

@article{quantization-1,
	title        = {Gptq: Accurate post-training quantization for generative pre-trained transformers},
	author       = {Frantar, Elias and Ashkboos, Saleh and Hoefler, Torsten and Alistarh, Dan},
	year         = 2022,
	journal      = {arXiv preprint arXiv:2210.17323}
}

@inproceedings{spin-quantization,
	title        = {SpinQuant: LLM Quantization with Learned Rotations},
	author       = {Liu, Zechun and Zhao, Changsheng and Fedorov, Igor and Soran, Bilge and Choudhary, Dhruv and Krishnamoorthi, Raghuraman and Chandra, Vikas and Tian, Yuandong and Blankevoort, Tijmen},
	year         = 2025,
	booktitle    = {The Thirteenth International Conference on Learning Representations}
}

@article{actaware-quantization,
	title        = {Awq: Activation-aware weight quantization for on-device llm compression and acceleration},
	author       = {Lin, Ji and Tang, Jiaming and Tang, Haotian and Yang, Shang and Chen, Wei-Ming and Wang, Wei-Chen and Xiao, Guangxuan and Dang, Xingyu and Gan, Chuang and Han, Song},
	year         = 2024,
	journal      = {Proceedings of machine learning and systems},
	volume       = 6,
	pages        = {87--100}
}

@inproceedings{bert-quantization,
	title        = {BinaryBERT: Pushing the Limit of BERT Quantization},
	author       = {Bai, Haoli and Zhang, Wei and Hou, Lu and Shang, Lifeng and Jin, Jin and Jiang, Xin and Liu, Qun and Lyu, Michael and King, Irwin},
	year         = 2021,
	booktitle    = {Proceedings of the 59th Annual Meeting of the Association for Computational Linguistics and the 11th International Joint Conference on Natural Language Processing (Volume 1: Long Papers)},
	pages        = {4334--4348}
}

@inproceedings{deberta,
	title        = {DeBerta: Decoding-Enhanced BERT with disentangled attention},
	author       = {He, Pengcheng and Liu, Xiaodong and Gao, Jianfeng and Chen, Weizhu},
	booktitle    = {International Conference on Learning Representations},
	yeat         = 2021
}

@inproceedings{debertav3,
	title        = {DeBERTaV3: Improving DeBERTa using ELECTRA-Style Pre-Training with Gradient-Disentangled Embedding Sharing},
	author       = {He, Pengcheng and Gao, Jianfeng and Chen, Weizhu},
	year         = 2023,
	booktitle    = {The Eleventh International Conference on Learning Representations}
}

@article{roberta,
	title        = {Roberta: A robustly optimized bert pretraining approach},
	author       = {Liu, Yinhan and Ott, Myle and Goyal, Naman and Du, Jingfei and Joshi, Mandar and Chen, Danqi and Levy, Omer and Lewis, Mike and Zettlemoyer, Luke and Stoyanov, Veselin},
	year         = 2019,
	journal      = {arXiv preprint arXiv:1907.11692}
}

@article{svd,
	title        = {A generalization of the Eckart-Young-Mirsky matrix approximation theorem},
	author       = {Hoffman, Alan},
	year         = 1987,
	journal      = {Linear Algebra and its applications},
	publisher    = {Elsevier},
	volume       = 88,
	pages        = {317--327}
}

@article{transformer,
	title        = {Attention is all you need},
	author       = {Vaswani, Ashish and Shazeer, Noam and Parmar, Niki and Uszkoreit, Jakob and Jones, Llion and Gomez, Aidan N and Kaiser, {\L}ukasz and Polosukhin, Illia},
	year         = 2017,
	journal      = {Advances in neural information processing systems},
	volume       = 30
}

@article{dynabert,
	title        = {Dynabert: Dynamic bert with adaptive width and depth},
	author       = {Hou, Lu and Huang, Zhiqi and Shang, Lifeng and Jiang, Xin and Chen, Xiao and Liu, Qun},
	year         = 2020,
	journal      = {Advances in Neural Information Processing Systems},
	volume       = 33,
	pages        = {9782--9793}
}

@inproceedings{tinybert,
          title={Tinybert: Distilling bert for natural language understanding},
          author={Jiao, Xiaoqi and Yin, Yichun and Shang, Lifeng and Jiang, Xin and Chen, Xiao and Li, Linlin and Wang, Fang and Liu, Qun},
          booktitle={Findings of the association for computational linguistics: EMNLP 2020},
          pages={4163--4174},
          year={2020}
}

@article{distillation-survey,
	title        = {Knowledge distillation: A survey},
	author       = {Gou, Jianping and Yu, Baosheng and Maybank, Stephen J and Tao, Dacheng},
	year         = 2021,
	journal      = {International journal of computer vision},
	publisher    = {Springer},
	volume       = 129,
	number       = 6,
	pages        = {1789--1819}
}

@inproceedings{matrix-decomp,
	title        = {Compressing Pre-trained Language Models by Matrix Decomposition},
	author       = {Ben Noach, Matan and Goldberg, Yoav},
	year         = 2020,
	booktitle    = {Proceedings of the 1st Conference of the Asia-Pacific Chapter of the Association for Computational Linguistics and the 10th International Joint Conference on Natural Language Processing},
	publisher    = {Association for Computational Linguistics},
	address      = {Suzhou, China},
	pages        = {884–889},
	doi          = {10.18653/v1/2020.aacl-main.88},
	url          = {https://aclanthology.org/2020.aacl-main.88},	
	language     = {en}
}

@book{DLBook,
	title        = {Deep Learning},
	author       = {Ian Goodfellow and Yoshua Bengio and Aaron Courville},
	year         = 2016,
	publisher    = {MIT Press},
}

@inproceedings{features-low-rank-weights-not,
	title        = {Compressing transformers: features are low-rank, but weights are not!},
	author       = {Yu, Hao and Wu, Jianxin},
	year         = 2023,
	booktitle    = {Proceedings of the AAAI Conference on Artificial Intelligence},
	volume       = 37,
	number       = 9,
	pages        = {11007--11015}
}

@article{glue-benchmark,
	title        = {GLUE: A Multi-Task Benchmark and Analysis Platform for Natural Language Understanding},
	author       = {Wang, Alex and Singh, Amanpreet and Michael, Julian and Hill, Felix and Levy, Omer and Bowman, Samuel R},
	year         = 2018,
	journal      = {EMNLP 2018},
	pages        = 353
}

@article{mobilebert,
      title={LightMobileBert: A secondary lightweight model based on MobileBert},
      author={Chen, Deguang and Zhou, Jie},
      journal={Journal of Intelligent \& Fuzzy Systems},
      volume={44},
      number={2},
      pages={2117--2129},
      year={2023},
      publisher={SAGE Publications Sage UK: London, England}
}

@article{distilbert,
  title={DistilBERT, a distilled version of BERT: smaller, faster, cheaper and lighter},
  author={Sanh, Victor and Debut, Lysandre and Chaumond, Julien and Wolf, Thomas},
  journal={arXiv preprint arXiv:1910.01108},
  year={2019}
}

@article{ptq,
  title={Optimizing spatial shift point-wise quantization},
  author={Kim, Eunhui and Lee, Kyong-Ha and Sung, Won-Kyung},
  journal={IEEE Access},
  volume={9},
  pages={68008--68016},
  year={2021},
  publisher={IEEE}
}

@inproceedings{qat,
  title={Efficientqat: Efficient quantization-aware training for large language models},
  author={Chen, Mengzhao and Shao, Wenqi and Xu, Peng and Wang, Jiahao and Gao, Peng and Zhang, Kaipeng and Luo, Ping},
  booktitle={Proceedings of the 63rd Annual Meeting of the Association for Computational Linguistics (Volume 1: Long Papers)},
  pages={10081--10100},
  year={2025}
}

@article{tsvd,
  title={Augmented implicitly restarted Lanczos bidiagonalization methods},
  author={Baglama, James and Reichel, Lothar},
  journal={SIAM Journal on Scientific Computing},
  volume={27},
  number={1},
  pages={19--42},
  year={2005},
  publisher={SIAM}
}

@article{rsvd,
  author = {Halko, N. and Martinsson, P. G. and Tropp, J. A.},
  title={Finding structure with randomness: Probabilistic algorithms for constructing approximate matrix decompositions},
  journal={SIAM review},
  volume={53},
  number={2},
  pages={217--288},
  year={2011},
  publisher={SIAM}
}
